\documentclass[twoside]{article}

\newcommand{\norm}[1]{\left\lVert#1\right\rVert}

\newcommand{\err}{\text{err}}

\newcommand{\Ps}{P_{\scriptscriptstyle\text{src}}}
\newcommand{\Pt}{P_{\scriptscriptstyle\text{trg}}}

\newcommand{\Psrc}{P_{\scriptscriptstyle\text{src}}}

\newcommand{\Pte}{P_{\scriptscriptstyle\text{test}}}
\newcommand{\Dte}{{D}_{\scriptscriptstyle\text{test}}}

\newcommand{\D}{D}
\newcommand{\Pun}{P_{\scriptscriptstyle\text{ulb}}}
\newcommand{\Dun}{{D}_{\scriptscriptstyle\text{ulb}}}

\newcommand{\Pwarm}{P_{\scriptscriptstyle\text{warm}}}
\newcommand{\Dwarm}{{D}_{\scriptscriptstyle\text{warm}}}

\newcommand{\Pmed}{P_{\scriptscriptstyle\text{med}}}
\newcommand{\Pss}{P_{\scriptscriptstyle\text{ss}}}

\newcommand{\thetasm}{\theta_{\scriptscriptstyle{s \shortrightarrow m}}}
\newcommand{\thetast}{\theta_{\scriptscriptstyle{s \shortrightarrow t}}}
\newcommand{\thetamt}{\theta_{\scriptscriptstyle{m \shortrightarrow t}}}
\newcommand{\thetaut}{\theta_{\scriptscriptstyle{u \shortrightarrow t}}}
\newcommand{\thetawt}{\theta_{\scriptscriptstyle{w \shortrightarrow t}}}
\newcommand{\rsm}{r_{\scriptscriptstyle{s \shortrightarrow m}}}
\newcommand{\rst}{r_{\scriptscriptstyle{s \shortrightarrow t}}}
\newcommand{\rmt}{r_{\scriptscriptstyle{m \shortrightarrow t}}}
\newcommand{\hatrmt}{\hat{r}_{\scriptscriptstyle{m \shortrightarrow t}}}

\newcommand{\argmin}{\text{argmin}}

\newcommand{\abs}[1]{\left|#1\right|}
\newcommand{\brck}[1]{\left(#1\right)}

\newcommand{\brcksq}[1]{\left[#1\right]}

\newcommand{\expc}[1]{\mathbb{E}_{#1}}
\newcommand\myeq{\mkern1.5mu{=}\mkern1.5mu}

\newcommand{\batchedalls}{Batched MALLS}
\newcommand{\alls}{MALLS}
\newcommand{\coloneqqq}{\mkern2.0mu{\coloneqq}\mkern2.0mu}
\usepackage{microtype}
\usepackage{todonotes}
\usepackage{amsmath}
\usepackage{dsfont}
\usepackage{amssymb}
\usepackage{bbm}
\usepackage{amsthm}
\usepackage{wrapfig}
\usepackage{algorithm}
\usepackage{algorithmic}
\usepackage{eqparbox}
\usepackage{hyperref}

\usepackage[ruled,vlined,algo2e]{algorithm2e}
\usepackage{enumitem,kantlipsum}
\newtheorem*{theorem*}{Theorem}
\newtheorem{theorem}{Theorem}
\newtheorem{lemma}{Lemma}

\usepackage[accepted]{aistats2021}
\usepackage{breqn}
\usepackage{stmaryrd}
\usepackage{mathtools}
\begin{document}

%

%

\twocolumn[

\aistatstitle{Active Learning under Label Shift}

\aistatsauthor{ Eric Zhao \And Anqi Liu \And Anima Anandkumar \And Yisong Yue }

\aistatsaddress{ California Institute of Technology } ]

\begin{abstract}
We address the problem of active learning under \textit{label shift}: when the class proportions of source and target domains differ.
We introduce a ``medial distribution'' to incorporate a tradeoff between importance weighting and class-balanced sampling and propose their combined usage in active learning.
Our method is known as Mediated Active Learning under Label Shift (\alls{}). It balances the bias from class-balanced sampling and the variance from importance weighting.
We prove sample complexity and generalization guarantees for \alls{} which show active learning reduces asymptotic sample complexity even under arbitrary label shift.
We empirically demonstrate \alls{} scales to high-dimensional datasets and can reduce the sample complexity of active learning by 60\% in deep active learning tasks.

\end{abstract}

\section{Introduction}
\textbf{Label Shift} 
In many real-world applications, the target (testing) distribution of a model can differ from the source (training) distribution.
Label shift arises when class proportions differ between the source and target, but the feature distributions of each class do not.
For example, the problems of bird identification in San Francisco (SF) versus New York (NY) exhibit label shift.
While the likelihood of observing a snowy owl may differ, snowy owls should look similar in New York and San Francisco.
The well-known class-imbalance problem is a specific form of label shift where the target label distribution is uniform but the source is not.

\textbf{Active Learning under Label Shift}
Label shift poses a problem for active learning in the real world.
For example, we can train a bird classifier for New York by labeling bird images off Google.
However, due to the label shift between New York and Google Images, naive active learning algorithms will fail to collect data on birds relevant in New York.
The correction of minority underrepresentation in computer vision datasets \cite{yang_towards_2020} similarly poses an active learning under label shift problem.
Proper label shift correction must be incorporated into active learning techniques to avoid inefficient and biased data collection. 

\begin{figure}[t]
\centering
  \includegraphics[height=4cm, trim={1.2cm, 3.5cm, 2cm, 5.2cm},clip]{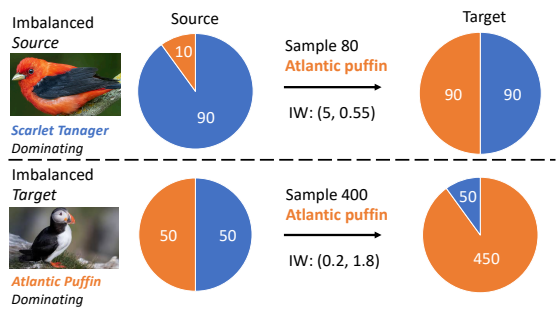}
  \caption{
  Extreme label shift examples of binary classification on $100$ datapoints.
  The arrows in the example illustrate two ways of correcting \textit{imbalanced source} and \textit{imbalanced target}. It requires larger importance weights to correct \textit{imbalanced source} than \textit{imbalanced target}, but correcting
  \textit{imbalanced target} requires more additional samples than \textit{imbalanced source}. A uniform \textit{medial distribution} can decompose any label shift into an \textit{imbalanced source} and \textit{imbalanced target}.}
\label{fig:imbalancesettings}
\end{figure}

\textbf{Importance Weighting \& Subsampling}
There are two ways to correct label shift, as shown in the two extreme cases depicted in Figure \ref{fig:imbalancesettings}. The arrows demonstrate the required additional samples and importance weights for the correction of \textit{imbalanced source} and \textit{imbalanced target}. Importance weighting can correct for label shift with rigorous theoretical guarantees \cite{lipton_detecting_2018,azizzadenesheli_regularized_2019}.
However, under large label shift, the estimation and use of importance weights result in high variance. 
Class-balanced sampling (subsampling) can also correct for label shift and, although lacking strong theoretical guarantees, is practical and effective \cite{aggarwal_active_2020}.
However, in active learning settings, subsampling is imprecise as only label predictions---not true labels---can be used to assign subsampling probabilities to unlabeled datapoints.

\begin{figure}[tbh]
\centering
\begin{tabular}{cccc}
   \fbox{\includegraphics[height=1.3cm]{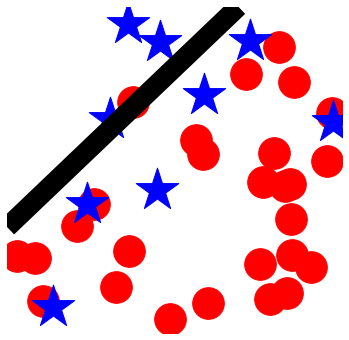}} &
   \fbox{\includegraphics[height=1.3cm]{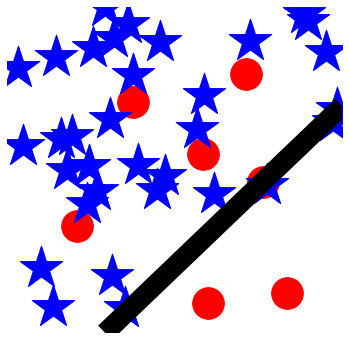}}&
   \fbox{\includegraphics[height=1.3cm]{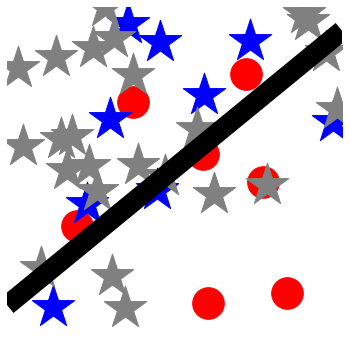}}&
   \fbox{\includegraphics[height=1.3cm]{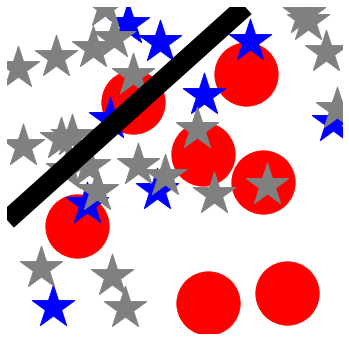}}\\
   \tiny{Target Data} & \tiny{Source Data} &\tiny{Subsampled} & \tiny{Weighted}
\end{tabular}
\caption{Noisy linear classification of stars and circles with a star-heavy source and circle-heavy target. Black lines depict the empirical risk minimizer (ERM). Ignored data are light grey. Our proposed algorithms first subsample a \textit{medial distribution}---in this case, equal parts circle and star---then importance weighting produces a  circle-dominant ERM. See sec. \ref{sec:alls}-\ref{sec:exps} for details.}
\label{fig:jointcorrection}
\end{figure}

In this paper, we answer the question: how can we use both importance weighting and subsampling for active learning---and how much should we use each? We answer this question by introducing a \textit{medial distribution} (Figure \ref{fig:jointcorrection}).
Rather than active learning on datapoints from the source distribution, datapoints are instead sampled from a \textit{medial distribution} by subsampling.
Importance weighting corrects the remaining label shift between the \textit{medial} and target distributions. 

\textbf{Our contributions}:
\begin{enumerate}[leftmargin=*,noitemsep,topsep=0pt]
\item Introduction of a \textit{medial distribution} to describe a bias-variance trade-off in label shift correction.
\item Mediated Active Learning under Label Shift (MALLS), a principled algorithm with theoretical guarantees even under label shift.
\item A batched variant of \alls{} for practitioners which integrates best practices and uncertainty sampling.
\end{enumerate}

Aggressive use of subsampling can reduce the need for, and thus variance of, importance weighting.
However, subsampling also introduces bias from its use of proxy labels.
We derive a bias-variance tradeoff that formalizes this trade-off and guides algorithm design.
In particular, we show subsampling can mitigate the effect of label shift on importance weighting variance and label complexity---but at the cost of introducing bias. 
We further propose a choice of uniform medial distribution, as we illustrate in Figure  \ref{fig:imbalancesettings}.

To the best of our knowledge, \alls{} is the first active learning framework for general \textit{label shift} settings.
We also derive label complexity and generalization PAC bounds for \alls{}, the first such guarantees for this setting.
We present experiments of \alls{} which corroborate our theoretical insights into the trade-off between importance weighting and subsampling.
In particular, batched \alls{} improves the sample efficiency of popular active learning algorithms by up to 60\% in the CIFAR10, CIFAR100 \cite{krizhevsky_learning_2009}, and NABirds datasets \cite{van_horn_building_2015}. We share the source code for the implementation of our method in this repository: \url{https://github.com/ericzhao28/alls}.

\section{Related Works}
\label{section:related}

\textbf{Active Learning}
Active learning has been investigated extensively from both theoretical and practical perspectives.
Disagreement-based active learning and its variants  enjoy rigorous learning guarantees and focus on the stream-based active learning setting \cite{hanneke_bound_2007, hanneke_activized_2011,balcan_agnostic_2009,hanneke_theory_2014,beygelzimer_agnostic_2010, krishnamurthy_active_2019}.
On the other hand, uncertainty sampling techniques are popular practical algorithms which have been successfuly applied to natural language processing \cite{shen_deep_2018}, computer vision \cite{yang_multi-class_2015}, and even robotics \cite{choudhury_bayesian_2020}.
We can incorporate our medial distribution design principle to arrive at both a streaming disagreement-based \alls{} approach, as well as a \batchedalls{} for uncertainty sampling.

\textbf{Distribution Shift}
General domain adaptation theory \cite{ben-david_analysis_2007,ben-david_theory_2010, cortes_learning_2010,cortes_domain_2014} looks at joint distribution shift.
Covariate shift is the most popular refinement of joint distribution shift \cite{shimodaira_improving_2000,gretton_covariate_2009,sugiyama_covariate_2007}.
However, density estimation for joint distribution shift or covariate shift is challenging due to the high-dimension nature of input features in many applications \cite{sugiyama_density_2012,tsuboi_direct_2009,yamada_relative_2011}.
The label shift setting is comparatively less popular, but has received increased attention in recent years \cite{lipton_detecting_2018,azizzadenesheli_regularized_2019,garg_unified_2020}.
Density estimation under label shift is comparatively simpler than under covariate shift: label spaces are simpler and often finite \cite{lipton_detecting_2018}.

\textbf{Active Learning under Distribution Shift}
Active learning \cite{rai_domain_2010, matasci_svm-based_2012,deng_active_2018,su_active_2019} has been studied under joint distribution shift and covariate shift.
Existing literature, which sometimes term the problem ``active domain adaptation'', rely on heuristics for correcting joint distribution shift \cite{chan_domain_2007,rai_domain_2010} or build on the assumption of \textit{covariate shift} \cite{saha_active_2011,yan_active_2018,chattopadhyay_joint_2013}.
While active learning with a covariate-shifted warm start guarantees label complexity bounds, it requires importance weights known a priori \cite{yan_active_2018}.
Label shift is a particularly difficult setting as, unlike covariate shift, label shift cannot be estimated from unlabeled data.

With few exceptions \cite{huang_transfer_2016}, existing literature assume active learners can query datapoints from the test domain (our \textit{canonical label shift} setting).
To the best of our knowledge, \alls{} provides the first guarantees for where test data is limited or labels cannot be queried in the test domain. 

The closest existing work to active learning under label shift is active learning for imbalanced data \cite{aggarwal_active_2020, lin_active_2018}, which can be formalized as an instance of label shift with a uniform test distribution. While existing work have proposed useful heuristics like diverse sampling and class-balanced sampling, theoretical results are scarce.

\section{Preliminaries}
\textbf{Active Learning under Distribution Shift}
In an active learning problem, a learner $L$ actively collects a labeled dataset $S$ with the goal of maximizing the performance of the hypothesis $h \in H$ learned from $S$.
$m$ labeled datapoints sampled from some distribution $\Pwarm$ initially populate $S$ and constitute the  ``warm start'' dataset $\Dwarm$.
$L$ samples unlabeled datapoints $\Dun$ from some distribution $\Pun$, and may select up to $n$  for labeling and appending to $S$.
The learned hypothesis $h$  is evaluated on a test distribution $\Pte$.
Traditional active learning assumes,
\begin{align}
    \Pun = \Pwarm = \Pte.
\end{align}
In contrast, \textit{active domain adaptation} does not assume the warm start is sampled from the test distribution:
\begin{align}
    \Pun = \Pte \text{ and } \Pwarm \neq \Pte.
\end{align}
This setting, which we term \textit{canonical label shift}, is well-studied but assumes active learning occurs in the test domain.
We address the more challenging \textit{general label shift} setting (Figure \ref{fig:shiftsettings}) which drops this assumption.
In the worst case, all distributions could be different:
\begin{align}
\Pwarm \neq \Pun, \Pwarm \neq \Pte , \Pun \neq \Pte .
\end{align}
For instance, the problem of creating a bird classifier for New York by actively labeling data off Google is \textit{general label shift}.
This setting has received comparatively little attention despite its practical relevance \cite{huang_transfer_2016}: there may be a scarcity of unlabeled target data or practical issues with labeling target data, such as patient privacy or ownership rights.

\textbf{Label Shift}
\label{sec:rlls}
The distribution shift problem concerns training and evaluating models on different distributions, termed the source ($\Ps$) and target ($\Pt$) respectively.
We refer to a source and target in the abstract.
For instance, in the \textit{canonical label shift} setting, the source is the warm start $\Ps = \Pwarm$, and the target is the test $\Pt = \Pte$.
Unlike covariate shift, which assumes the underlying distribution shift arises solely from a change in the input distribution while conditional label probabilities are unaffected\footnote{We abuse notation and define $P(\cdot)$ as $P(x) \coloneqqq P(X \myeq x)$ or $P(y) \coloneqqq P(Y \myeq y)$ depending on the context.},
label shift assumes distribution shift arises solely from a change in label marginals:
\begin{align}
    \Pt(Y) \neq \Ps(Y),  \Pt(X|Y) = \Ps(X|Y).
\end{align}

\begin{figure}[t]
\centering
   \includegraphics[height=3cm, trim={1.8cm, 4cm, 1.5cm, 8.7cm},clip]{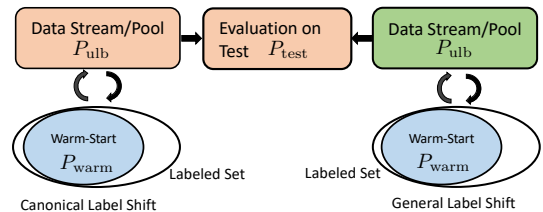}
   \caption{Diagram of active learning under label shift settings. The \textit{canonical label shift} actively samples from the test distribution. The \textit{general label shift setting} actively samples from elsewhere. Identical colors indicate identical distributions.}
\label{fig:shiftsettings}
\end{figure} 


\textbf{Importance Weighting (IW)}
Importance weighting is a straight-forward solution to label shift.
Weighting datapoints by likelihood ratio 
produces asymptotically unbiased importance weighted estimators.
\begin{align}
    \frac{1}{n} \sum_{i=1}^n \frac{\Pt(y_i)}{\Ps(y_i)} f(x_i, y_i) \notag
    & \rightarrow \expc{x, y \sim \Ps} \brcksq{\frac{\Pt(y)}{\Ps(y)} f(x, y)} \\
    & = \expc{x, y \sim \Pt} \brcksq{f(x, y)}.
\end{align}
Following existing label shift literature, we restrict our learning problems to those with a finite $k$-class label space.
We can estimate these importance weights with only labeled data from the source distribution, unlabeled data from the target distribution, and a blackbox hypothesis $h_0$ \cite{lipton_detecting_2018}.
Let $C_h$ denote the confusion matrix for hypothesis $h$ on $\Ps$ where
$\expc{}[C_h[i, j]] \coloneqqq \Ps(h(X) \myeq y^{(i)}, Y \myeq y^{(j)})$
and $q_h$ denote a $k$-vector with $q_h[i] \coloneqqq \Pt(h(X) \myeq y^{(j)})$.
Assuming for all labels $\forall y: \Pt(y) >0 \implies \Ps(y)>0$, \cite{lipton_detecting_2018} shows importance weights $r$ are,
\begin{align}
    r \coloneqq \frac{\Pt (y)}{\Ps (y)} = C^{-1}_{h_0} q_{h_0}.
    \label{eq:impweight}
\end{align}
For instance, Regularized Learning under Label Shift (RLLS) \cite{azizzadenesheli_regularized_2019} finds $r$ through convex optimization of:
\begin{align}
  C^{-1}_{h_0} q_{h_0}
   \approx \argmin_{r} \norm{C_{h_0} r - q_{h_0}}_2 + \lambda \norm{r - 1}_2,
   \label{eq:bbse}
\end{align}
where $\lambda$ is some regularization constant.

\paragraph{Class-balanced Sampling (Subsampling)}
A popular heuristic for addressing class imbalance in active learning is adjusting the probability of labeling a datapoint by its predicted label \cite{yang_ensemble-based_2010,park_improved_2011}. 
Traditionally, class-balanced sampling aims to ensure equal representation of each label and can be framed as a form of label shift with a uniform target label distribution.
We generalize class-balanced sampling to general label shift problems with potentially non-uniform targets, a practice we term \textit{subsampling}.
We now describe two methods of subsampling. In these examples, we subsample a user-defined distribution $\Pmed$ from a source $\Ps$ using predictor $\phi$ for predicting proxy labels.
\begin{enumerate}
    \item \textbf{Subsampling with a filter $\Pss$}, where $\Pmed(y) \propto \Pss(y \myeq \phi(x)) \Ps(y \myeq \phi(x))$. Repeat until a sample is yielded: sample datapoint $x$ from $\Ps$ and, with probability $\Pss(Y \myeq \phi(x))$, yield $x$.
    \item \textbf{Subsampling with the target $\Pmed$}. Collect $N$ datapoints from $\Ps$ into a buffer $S'$, where $N$ is large. For each label $y \in Y$, randomly add $N \Pmed(y)$ datapoints from $\{x \in S' \mid \phi(x) = y\}$ into a buffer $S$. To sample from $\Pmed$, draw from $S$. 
\end{enumerate}
While in finite settings only the former yields IID samples, the two are identical in the limit by the law of large numbers.
Since subsampling strictly concerns proxy labels as thus does not require labeled samples, we assume subsampling occurs at the limit and use the two interchangeably. 

In the expectation, 
subsampling is equivalent to importance weighting with proxy labels predicted by $\phi$:
\begin{align}
    \expc Q\left[\frac{1}{n} \sum_{i=1}^n Q_i f (x_i, y_i)\right] \notag
    & =
    \frac{1}{n} \sum_{i=1}^n \frac{\Pt(y_i \myeq \phi(x_i))}{\Ps(y_i \myeq \phi(x_i))} f(x_i, y_i),
\end{align}
where $Q_i \in \{0, 1\}$ is an indicator variable for whether the $i$th datapoint is subsampled and has conditional expectation $\expc{}[Q_i \mid y_i] = P_{\text{ss}}(y_i)$.

\section{Medial Distribution}
In this section, we propose the concept of a \textit{medial distribution}.
We conceptually frame subsampling as the importance sampling of an alternative distribution from the source distribution.
We term this alternative distribution the \textit{medial distribution} $\Pmed$.
As we will show, $\Pmed$ mediates a trade-off between subsampling and importance weighting (IW).

\paragraph{IW-Subsampling Trade-off}
In this section, we adopt domain adaptation notation and denote source, medial and target distributions as $\Ps, \Pmed, \Pt$.
Let $\rst \coloneqqq \Pt(y) / \Ps(y)$ denote the importance weights which shift the source to the target.
Similarly, let  $\rsm \coloneqqq \Pmed(y) / \Ps(y)$ and $\rmt \coloneqqq \Pt(y) / \Pmed(y)$ denote the importance weights \textit{to} and \textit{from} the medial distribution.
Note that $\Pmed(y), \Ps(y), \Pt(y)$ denote the likelihood of a ground-truth label $y$.
Estimated weights are accented with a hat: $\hat{r}$.
We follow \cite{lipton_detecting_2018} and formalize label shift magnitude as $\norm{\theta}$: some norm of $\theta \coloneqqq r - \mathbf{1}$, usually the L2 norm $\norm{\cdot}_2$.
A large $\norm{\theta}$ corresponds to a larger label shift and harder learning problem.
$\norm{\thetasm{}}$ is the amount of label shift corrected by subsampling and $\norm{\thetamt{}}$ is the amount corrected by importance weighting.
We can analyze this \textit{medial distribution} trade-off by introducing the following bound on the accuracy of empirical loss estimates under label shift where subsampling and importance weighting are used.
This theorem is a modification of a common error bound for offline supervised learning under label shift.
%

\begin{theorem}
Let $\Delta$ denote the subsampling and importance weighting (trained on $n$ datapoints) estimation error of the empirical loss of $N$ datapoints:
\begin{align}
    \Delta \coloneqqq \frac{1}{N} \sum_{i=1}^N \brck{r_i \Pss(y_i) - \hat{r}_i \Pss(h(x_i)) } \ell(h(x_i), y_i),
\end{align}
where $\ell: Y \times Y \rightarrow [0, 1]$ is a loss function.
With probability $1 - 2 \delta$, for all $n \geq 1$:
\begin{align}
| \Delta |
& \leq \mathcal{O} \left(
    \frac{2 }{\sigma_{\min}} \left(
        \norm{\thetamt}_2
        \sqrt{
            \frac{\log \brck{\frac{n k}{\delta}}}{n}
        }
        \right. \right.
        +  \notag \nonumber \\ & \left. \left.
        \sqrt{
            \frac{\log \brck{\frac{ n}{\delta}}}{n}
        }
        + 
        \norm{\thetasm}_{\infty} 
        \text{err}(h_0, \rmt)
    \right)
\right),
\end{align}%
where $\sigma_{\min}$ denotes the smallest singular value of the confusion matrix and $\err(h_0, r)$ denotes the importance weighted $0/1$-error of a blackbox predictor $h_0$ on $\Psrc$.
\label{thm:1}
\end{theorem}

The error bound in theorem \ref{thm:1} shows that the use of subsampling versus importance weighting results in different error bounds with different trade-offs.
In particular, the trade-off lies between the first summand, $\norm{\thetamt}_2
        \sqrt{
            \frac{\log \brck{\frac{n k}{\delta}}}{n}
        }$, and the third summand, $\norm{\thetasm}_{\infty} \text{err}(h_0, \rmt) $.
The former term corresponds to the error introduced by the use of importance weights---in particular, the variance that arises from importance weight estimation.
This variance is sensitive to the magnitude of the ground-truth importance weights, $\norm{\rmt}_2$.
Recall that in our medial distribution framework, importance weights correct the label shift between $\Pmed{}$ and $\Pt{}$.
The latter term corresponds to the subsampling estimation error---in particular, the bias introduced by the use of proxy labels for data weighting.
This bias is sensitive to the magnitude of subsampling, $\norm{\thetasm}_{\infty}$, and the accuracy of the blackbox hypothesis $\text{err}(h_0, \rmt)$.
Hence, subsampling mitigates sensitivity to label shift magnitude by splitting the norm of total label shift, $\norm{\thetast}$, into the sum of two factors which scale with $\norm{\thetasm}$ and $\norm{\thetamt}$.


The key to addressing this bias-variance trade-off is choosing a medial distribution which balances the quality of the blackbox hypothesis $\text{err}(h_0, \rmt)$ and the label shift magnitude $ \norm{\rst}$.
In effect, subsampling turns a difficult label shift problem, requiring large importance weights $\rst$, into an easier label shift problem, with smaller importance weights $\rmt$. 

\paragraph{Uniform Medial Distribution}
\label{sec:medial}
We motivate a particular choice of medial distribution, a uniform label distribution, with an example.
Figure \ref{fig:imbalancesettings} depicts two fundamental label shift regimes which we term \textit{imbalanced source} and \textit{imbalanced target}.
\textit{Imbalanced target} requires smaller importance weights to correct than \textit{imbalanced source} and is hence more efficient for IW to correct.
\textit{Imbalanced source} requires fewer additional examples than \textit{imbalanced target} and is more efficient for subsampling to correct.
This holds more broadly.
Consider a binary classification problem with $n$ datapoints and two possible label distributions: balanced distribution $D_1$ with $n/2$ datapoints in each class, and imbalanced distribution $D_2$ with $n-1$ datapoints in the majority class.
Under \textit{imbalanced source}, where $\Psrc \coloneqqq D_2$ and $\Pt \coloneqqq D_1$, $n-2$ additional samples from the under-represented class are necessary for negating label shift.
Under \textit{imbalanced target}, where $\Psrc \coloneqqq D_1$ and $\Pt \coloneqqq D_2$,
$(n - 2)\frac{n}{2} \in \mathcal{O}(n^2)$ additional samples are necessary.

This suggests subsampling under \textit{imbalanced source} and importance weighting under \textit{imbalanced target}.
A uniform medial distribution decomposes every label shift problem into the two settings: subsample \textit{to} a uniform label distribution (imbalanced source) then importance weight away \textit{from} uniform (imbalanced target).
As we will show, this affords a convenient upper bounds on the sample complexity of active learning with a uniform medial distribution.
We will also show, experimentally, that uniform distributions serve as a reliable choice for medial distributions and perform similarly to ``square root” medial distributions that are optimal in simple cases, e.g., singleton $X$.

\section{Streaming MALLS}
\label{sec:alls}

 \begin{figure}[bt]
    \centering
    \begin{tabular}{c}
       \includegraphics[height=3.5cm, trim={3.5cm, 8cm, 5cm, 5cm},clip]{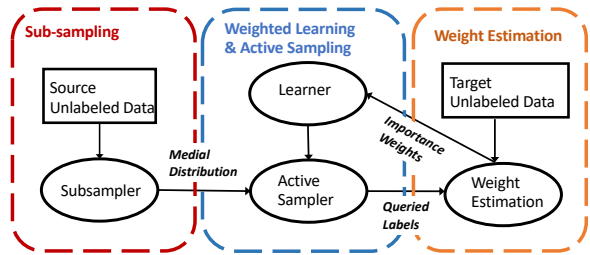}
       \end{tabular}
       \vspace{-0.10in}
     \caption{The \alls{} Algorithm.
     \alls{} consists of 3 routines: (1) class-balanced sampling from the unlabeled set, (2) actively querying for labels, and (3) importance weight estimation for correcting label shift.
    }
    \label{fig:allsflow}
\end{figure}

In this section, we present a streaming active learning algorithm: Mediated Active Learning under Label Shift (\alls{}). 
We analyze the generalization error and the label complexity of steaming \alls{} and validate the theory with experiments. We present a practical batched \alls{} approach in Sec. \ref{sec:exps}. 
We also open-source an implementation of \alls{}.

\begin{algorithm}[H]
\begin{algorithmic}
\begin{small}
\null
    \STATE  {\bfseries Input:} Warm start set ($\Dwarm$), unlabeled set ($\Dun$), test set ($\Dte$), active learning budget $n$, label shift budget $\lambda$, blackbox predictor $h_0$, medial distribution $\Pmed$, hypothesis class $\mathcal{H}$ \\
    \STATE  {\bfseries Initialize} the dataset $S \leftarrow$ warm start set;\\
    \STATE {\bfseries Subsample} the unlabeled set using $h_0$ to induce $\Pmed$.
    \STATE {\bfseries Estimate importance weights:}
    \STATE \quad Obtain $r$ with RLLS \cite{azizzadenesheli_regularized_2019} 
    \STATE \quad using $h_0$, unlabeled test data, and $\lambda n$ labeled
    \STATE \quad datapoints from the unlabeled set;\\
    \STATE  {\bfseries While} $|S| < n$ {
        \STATE \quad Calculate IWAL-CAL \cite{beygelzimer_agnostic_2010} sam
        \STATE \quad pling probability $P_t$ for $x_t$ using $S$ weighted by $r$;
        \STATE \quad Label and append $(x_t, y_t)$ to $S$ with probability $P_t$;\\
    }
    \STATE  {\bfseries Output:} $h_T \coloneqqq \text{argmin}_{h \in H} \err(h, r)$ where $\err$ is estimated on $S$.
\caption{Mediated Active Learning under Label Shift}
\label{alg:alls}
\end{small}
\end{algorithmic}
\end{algorithm}
\textbf{Proposed Algorithm}
We build on a popular importance-weighted agnostic active learning algorithm IWAL-CAL \cite{beygelzimer_agnostic_2010}.
We refer to IWAL-CAL as a subprocedure and defer its details to the Appendix.
IWAL-CAL takes as input a datapoint $x_t$ and returns a sampling probability $P_t$.
\alls{} modifies the computation of $P_t$ by applying importance weights to correct for label shift in empirical loss estimates.
Specifically, IWAL-CAL depends on estimating hypothesis loss on the actively labeled dataset $S$:
\begin{equation}
    \err(h) = \frac{1}{|S|} \sum_{t=1}^{|S|} \ell(h(x_t), y_t),
\end{equation}
where $x_i, y_i$ are drawn from $S$.
\alls{} instead computes empirical loss estimates as:
\begin{equation}
    \err(h, r) = \frac{1}{|S|} \sum_{t=1}^{|S|}  r(y_t) \ell(h(x_t), y_t),
\end{equation}
where $r(y_t)$ denotes an importance weight for datapoints of label $y_t$.
\alls{} computes these importance weights $r$ by calling a blackbox label shift estimator (e.g. BBSE \cite{lipton_detecting_2018}).
Our derivations use Regularized Learning under Label Shift (RLLS) \cite{azizzadenesheli_regularized_2019}.
Since label shift estimation algorithms require an independent holdout set for estimating importance weights, \alls{} estimates importance weights on a holdout set of $\lambda n$ labeled datapoints sampled from $\Pmed$ through subsampling.
\alls{} also adds subsampling as a preprocessing step to IWAL-CAL, re-using the blackbox hypothesis used in label shift estimation as a predictor.
Thus, instead of directly sampling points from $\Dun$, IWAL-CAL instead interacts with datapoints subsampled from $\Dun$ and distributed according to $\Pmed$.
We detail the high-level flow of \alls{} in Figure \ref{fig:allsflow} and provide pseudocode in Algorithm \ref{alg:alls}. 

\begin{figure}[bthp]
        \centering
         \setlength{\tabcolsep}{-0.2pt}
        \begin{tabular}{cc}
           \includegraphics[height=4cm]{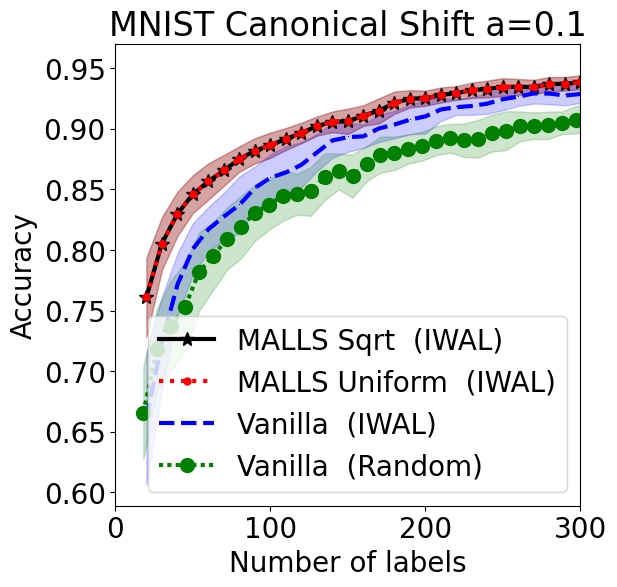}&
           \includegraphics[height=4cm]{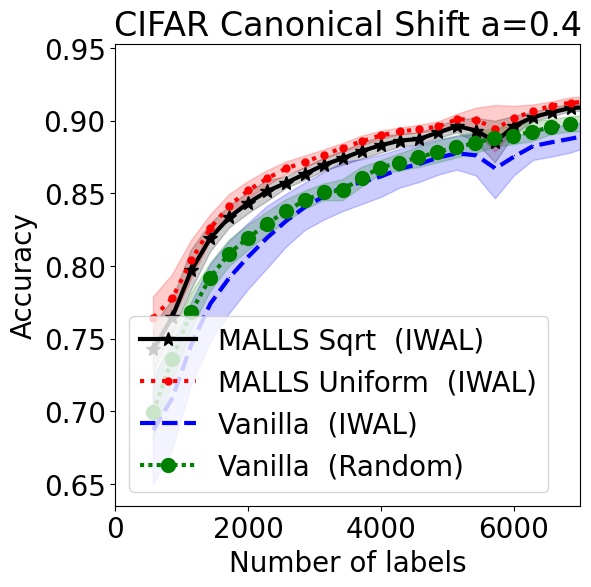}
       \end{tabular}
         \caption{ Average performance and 95\% confidence intervals on 10 runs of experiments on MNIST, 5 runs on CIFAR in a \textit{canonical label shift} setting  (defined in Preliminaries). Accuracy on (a) MNIST, (b) CIFAR. \alls{} leads to sample efficiency gains in both settings. A ``uniform'' medial performs on par with a ``square root'' medial.}
         \label{fig:alls}
\end{figure}
\begin{figure*}[t]
    \centering
     \setlength{\tabcolsep}{-0.2pt}
    \begin{tabular}{cccc}
       \includegraphics[height=4cm]{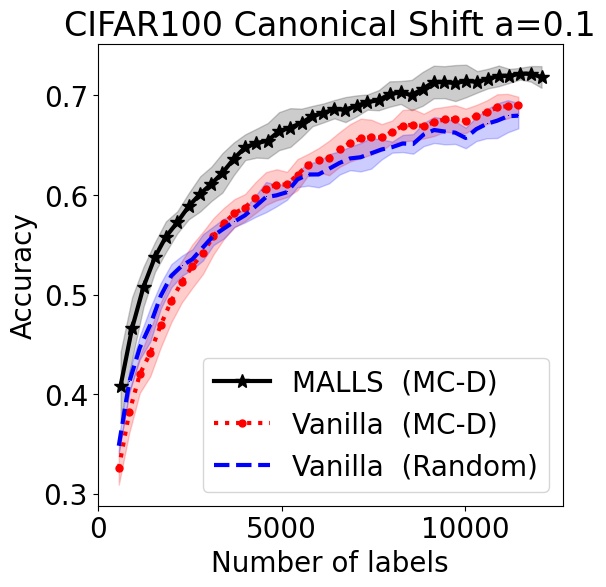}
       &
       \includegraphics[height=4cm]{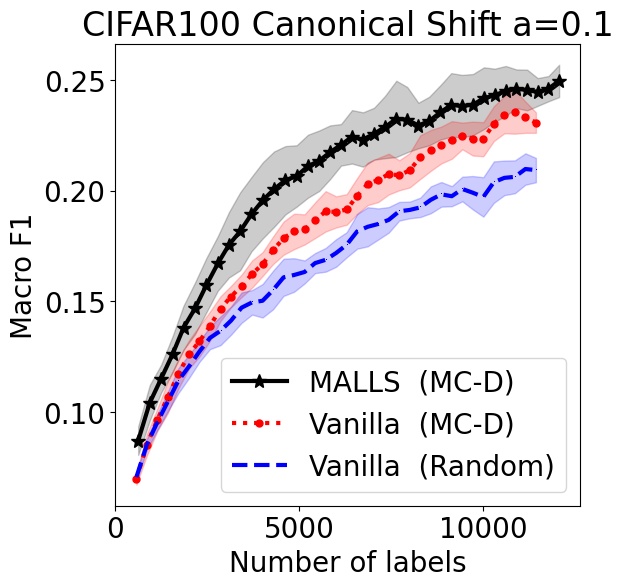}
       &
       \includegraphics[height=4cm]{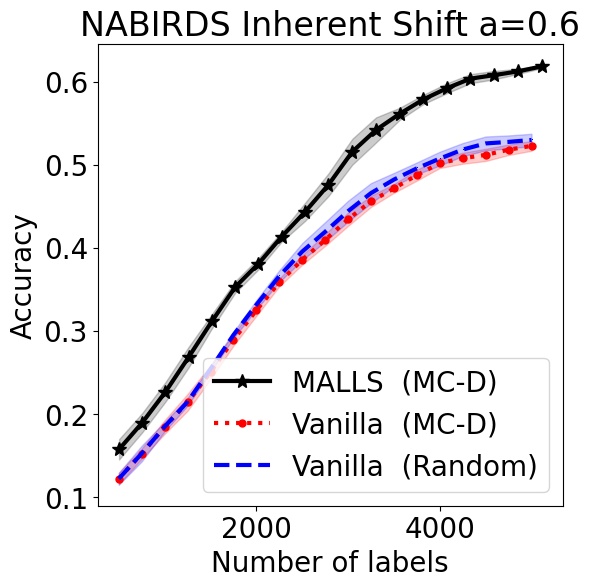} &
       \includegraphics[height=4cm]{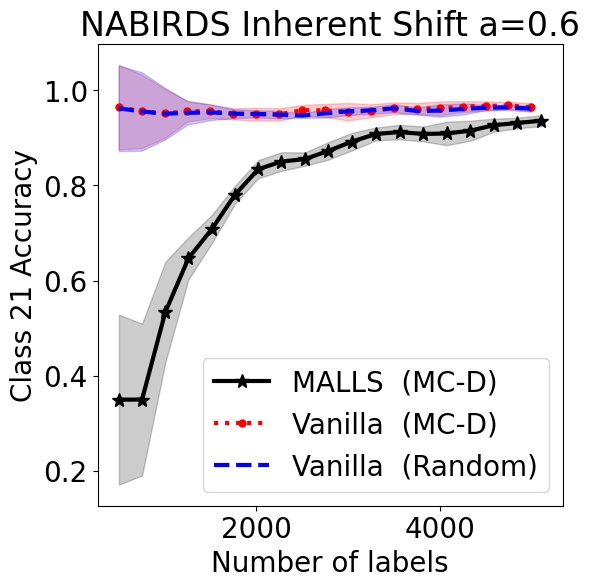}
       \\
       \tiny{(a)} &\tiny{(b)} & \tiny{(c)} & \tiny{(d)}
    \end{tabular}
   
    \caption{ Average performance and 95\% confidence intervals of 10 runs on \text{CIFAR100}, and 4 runs on NABirds.
    Plots (a)-(c) demonstrate \alls{} consistently improves accuracy and macro F1 scores.
    Plot (d) depicts the learning dynamics of \alls{} and verifies a suppression of the over-represented class (Class 21) during learning.
    }
    \label{fig:big}
\end{figure*}

 \begin{figure*}[t]
    \centering
     \setlength{\tabcolsep}{-0.2pt}
    \begin{tabular}{cccc}
      \includegraphics[height=4.1cm]{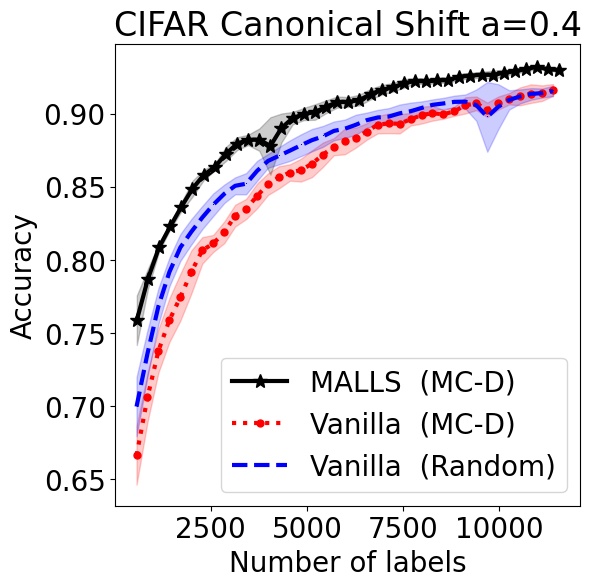}
      &
      \includegraphics[height=4.1cm]{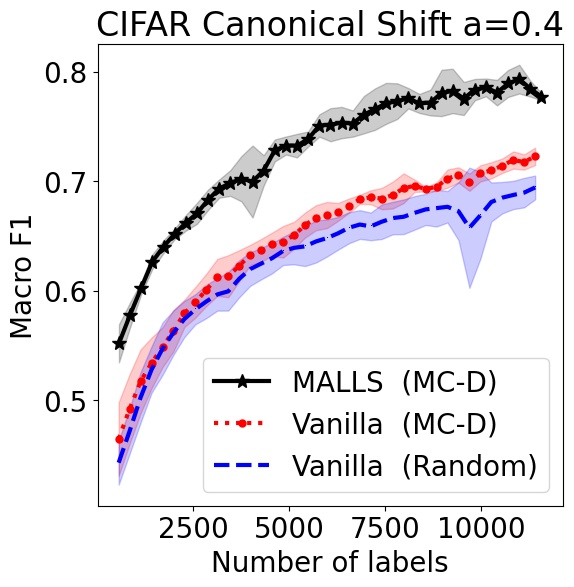}
       &
       \includegraphics[height=4.1cm]{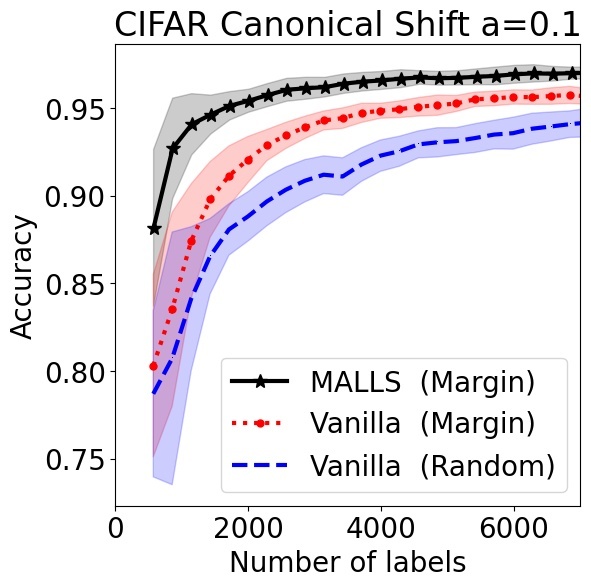}
      & \includegraphics[height=4.1cm]{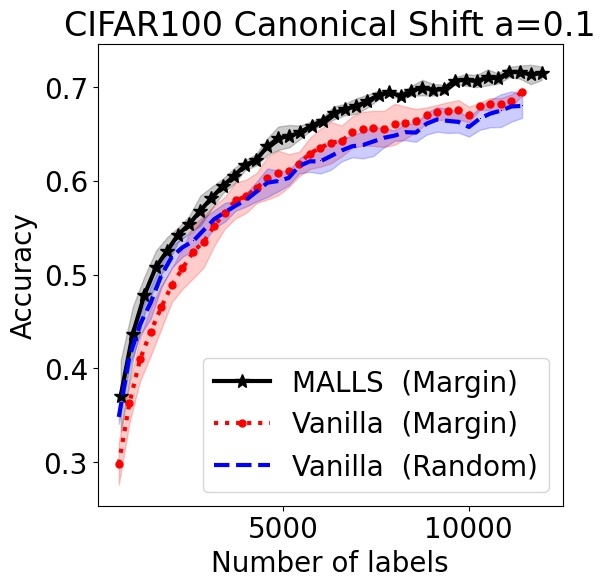}
      \\
      \tiny{(a)} &\tiny{(c)} & \tiny{(c)}  & \tiny{(d)}
    \end{tabular}
   
    \caption{ Average performance and 95\% confidence intervals of 10 runs on \text{CIFAR10} and CIFAR100.
    \alls{} consistently improves accuracy, macro F1, and weighted F1 scores.
    }
    \label{fig:cifar}
\end{figure*}

\textbf{Theoretical Analysis}
\label{sec:theory} 
We now analyze label complexity and generalization bounds for Algorithm \ref{alg:alls}.
In the canonical label shift setting, label shift naturally disappears asymptotically as the warm start dataset is diluted.
For the remainder of this section, we instead work in the more challenging \textit{general label shift} setting.
As the presence of warm start data is not particularly interesting in our analysis, we set the warm start budget $m = 0$ for reading convenience and defer the case where $m > 0$ to the Appendix for interested readers.
We also defer the case where the quantity of unlabeled test data is bounded to the Appendix.

The derivation of theoretical guarantees for \alls{} builds off our Theorem 1 and existing results from IWAL-CAL. The proof consists of two primary steps.
First, new deviation bounds are derived for IWAL-CAL to compensate for the additional variance introduced by subsampling and importance weighting.
Second, triangle inequalities plug in results from Section 3 on the bias-variance tradeoff.
The resulting deviation bound (see Appendix) yields the following guarantees for \alls{}.

\begin{theorem}
\label{thm:c_0}
With probability $> 1 - \delta$, for all $n \geq 1$,
\begin{small}
\begin{align}
   \err_Q(h_n) \leq  & \mathcal{O} \left( (1 + \frac{1}{\sigma_{\min}}) \norm{\rst}_{\infty} \err(h_0, \rmt)  \right. \notag \nonumber \\ & \left.
    +  \err_Q(h^*) +
    \sqrt{\frac{2 C_0 \log n}{n - 1}}
    + \frac{2 C_0 \log n}{n - 1} \right),
\end{align}%
\end{small}%
where $\err_Q$ denotes hypothesis error in the target domain, $n$ denotes observed datapoints including those not labeled or subsampled, and the constant $C_0$ is,
\begin{small}
\begin{align}
C_0 \in \mathcal{O} & \left(
    \frac{2}{\lambda \sigma_{\min}} \left(
        \norm{\thetamt}^2_2
        \log \brck{\frac{k}{\delta}} + \log \brck{\frac{1}{\delta}}
    \right)
\right. \notag \nonumber \\ &
\left.
    + \log \brck{\frac{|H|}{\delta}} (1 + \norm{\thetast}^2_2)
\right).
\end{align}
\end{small}
\end{theorem}

Our generalization bound differs from the original IWAL-CAL bound in two key aspects.
(1) The use of subsampling introduces bias related to the performance of the blackbox hypothesis: $\frac{1}{\sigma_{\min}} \err(h_0, \rmt)$.
(2) In the original IWAL-CAL algorithm $C_0 \in \mathcal{O} (\log\left(|H| / \delta \right)$.
However label shift inevitably introduces, to the constant $C_0$, a dependence on the number of label classes $k$ and label shift magnitudes $\norm{\thetast}_2^2$ and $\norm{\thetamt}_2^2$.
When the subsampling error is high, Theorem 2 shows importance weighting can be used alone to preserve a consistency guarantee even under \textit{general label shift}.

\begin{theorem}
\label{thm:smplcomplexity}
With high probability\footnote{Where $\delta < 2^{(-2e-1)/\norm{\rsm}_{\infty}}$.} at least $1 - \delta$, the number of labels queried is at most:
\begin{small}
\begin{align}
    & \mathcal{O} \left( 1 + \log \brck{\frac{1}{\delta}} +  \Theta \sqrt{C_0 \frac{n}{\norm{\rsm}_{\infty}} \log n} + \Theta C_0 \log^3 n + \lambda n \right. \nonumber \\
   & \left.
    + \Theta \cdot (n - 1) \cdot \brck{
        \frac{\err_Q(h^*)}{\norm{\rsm}_{\infty}} + (1 + \frac{1}{\sigma_{\min}}) \err(h_0, \rmt)
    } \right),
\end{align}%
\end{small}%
where 
$\Theta$
denotes the disagreement coefficient \cite{balcan_agnostic_2009}.
\end{theorem}
Subsampling effectively increases the noise rate of the underlying problem.
This increases the linear noise rate term $O(n)$ inevitable in agnostic active learning labeling complexities.
However, subsampling also reduces sample complexity by a factor of $\frac{1}{\norm{\rsm}_{\infty}}$.
Importance weighting introduces a new linear label complexity term $O(\lambda n)$.
This is used to collect a holdout set for label shift estimation.
Thus, when the blackbox hypothesis is bad and strong importance weighting is necessary, the sample complexity improvements of active learning are lost.
However, given a good blackbox hypothesis, the medial distribution can be set closer to the target (small $\norm{\thetamt}$) and $\lambda$ can be set small so \alls{} retains the sample complexity gains of active learning.

\textbf{Experiments}
We empirically validate \alls{} with experiments on synthetic label shift problems on the MNIST and CIFAR benchmark datasets.
These experiments employ a bootstrap approximation of IWAL-CAL recommended in \cite{beygelzimer_importance_2009} using a version space of 8 Resnet-18 models.
The blackbox hypothesis is obtained by training a standalone model on the warm start data split.
Random sampling and vanilla active learning (IWAL-CAL) are compared against \alls{} for two choices of medial distributions:
\begin{enumerate}[leftmargin=*,noitemsep,topsep=0pt]
    \item A ``square root'' medial distribution where $\rsm = \rmt = \sqrt{\rst}$. This is a bare optimization of the error tradeoff in Theorem 1.
    \item A ``uniform'' medial distribution motivated by Figure \ref{fig:imbalancesettings} and intuition of \textit{imbalanced sources/targets}.
\end{enumerate}
The results, shown in Figure \ref{fig:alls} demonstrate significant sample efficiency gains due to \alls{}, even when vanilla IWAL no longer beats random sampling.
Despite its simplicity, the performance of the uniform medial distribution is indistinguishable from the theoretically motivated ``square root'' medial distribution.

\section{\batchedalls{}}
\label{sec:exps}
 
We present a variant of \alls{} for practitioners which integrates best practices for scaling label shift estimation.
This variant, depicted in Algorithm \ref{alg:prac}, is a framework for batched active learning that supports any blackbox uncertainty sampling algorithm.

   

 \begin{figure*}[bthp]
        \centering
         \setlength{\tabcolsep}{-0.2pt}
        \begin{tabular}{cccc}
           \includegraphics[height=4cm]{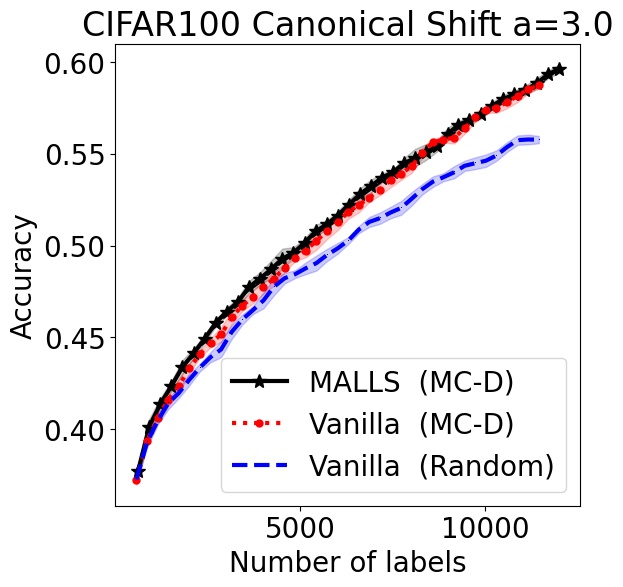}&
           \includegraphics[height=4cm]{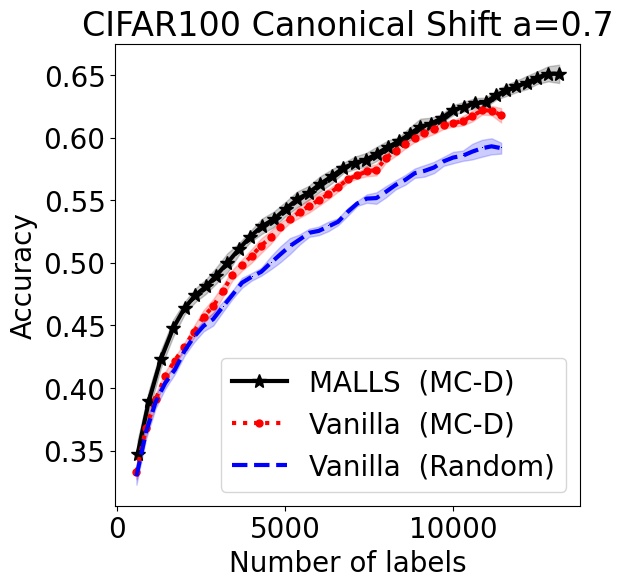}&
           \includegraphics[height=4cm]{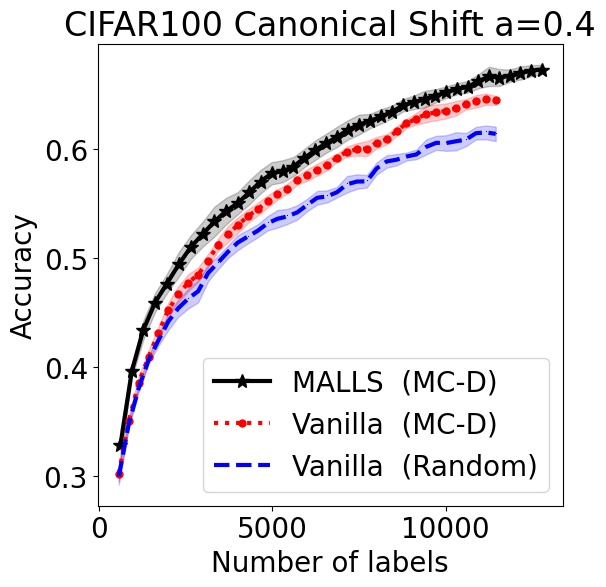}&
           \includegraphics[height=4cm]{figs/BasicBALD_cifar100_dirichlet_mix_warm_0,4_alpha_0,1_1000Accuracy.jpg}
       \end{tabular}
         \caption{ Average performance and 95\% confidence intervals on 10 runs of experiments on CIFAR100 in the \textit{canonical label shift} setting (defined in Preliminaries). In order of increasing label shift magnitude: (a), (b), (c), (d). MALLS performance gains scale by label shift magnitude. }
         \label{fig:alphas}
    \end{figure*}
    
\textbf{Best Practices}
\batchedalls{} incorporates five important techniques for scaling the real world practice of label shift correction. 
\begin{enumerate}[leftmargin=*,noitemsep,topsep=0pt]
    \item Forgo use of independent holdout sets and instead learn importance weights $r$ on the main dataset $S$. 
    \item Motivated by Theorem 1, \batchedalls{} uses the current active learning predictor for subsampling.
    \item Approximate subsampling 
    with a generalization of class-balanced sampling that is compatible with batch-mode active learning  \cite{aggarwal_active_2020}. 
    \item Apply importance weights during inference time. \batchedalls{} replaces the importance weighting of empirical loss estimates with posterior regularization, a practice closely related to the expectation-maximization algorithm in \cite{saerens_adjusting_2002}. 
    \item Use hypotheses learned with importance weights as blackbox predictors to learn better importance weights. 
    We term this iterative reweighting.
\end{enumerate}

\begin{algorithm}[H]
\begin{algorithmic}
\begin{small}
\null
    \STATE  {\bfseries Input:} Warm start set, unlabeled pool $\Dun$, test set,
                               number of batches $T$, medial distribution $\Pmed$,
                               uncertainty quantifier $\pi$, batch size $B$; hypothesis class $\mathcal{H}$\\
    \STATE  {\bfseries Initialize} the labeled dataset $S_0 \leftarrow $ the warm start set;\\
    \STATE {\bfseries Initialize} hypothesis $\phi$ by training on warm start set; \\
    \STATE  {\bfseries For} $t \in 1, \dots, T$ {
        \STATE \quad Find importance weights $r_t \leftarrow \text{RLLS}(S_t, \phi, \Pmed)$; \\
        \STATE \quad Train hypothesis $\phi$ on $S$ weighted by $r_t$; \\
        \STATE \quad  {\bfseries For} $y \in Y$ {
            \STATE \quad \quad Number of datapoints to collect $k \coloneqqq B \times \Pmed(y)$
            \STATE \quad \quad Find top-$k$ most uncertain datapoints of label $y$:\\
            \STATE \quad \quad \quad \quad $\D_y \coloneqqq \text{top-}k(\pi, \{x \in \Dun \setminus S_{t-1} \mid \phi(x) = y\})$\\
            \STATE \quad \quad Label and append the top-$k$ datapoints, $\D_y$, to $S_t$
        }
    }
    \STATE  {\bfseries Output:} $h_T = \text{argmin}\{\text{err}(h, S_T, r_T) : h \in \mathcal{H}\}$

\caption{\batchedalls{}}
\label{alg:prac}
\end{small}
\end{algorithmic}
\end{algorithm}

 \begin{figure*}[bthp]
        \centering
        \setlength{\tabcolsep}{-0.5pt}
        \begin{tabular}{cccc}
           \includegraphics[height=4cm]{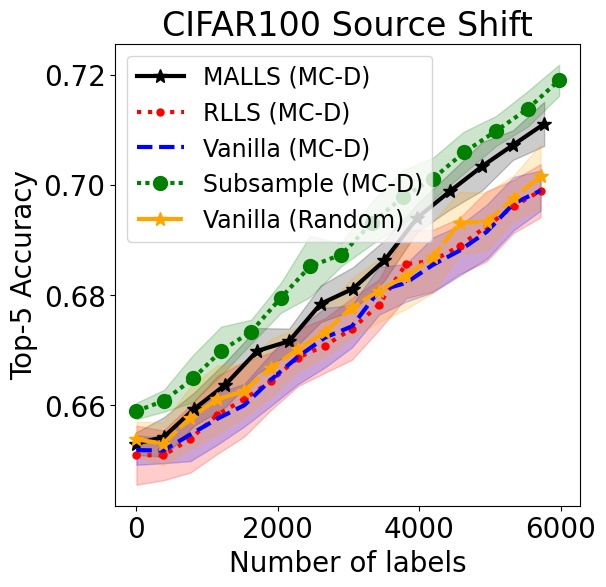}
           &
           \includegraphics[height=4cm]{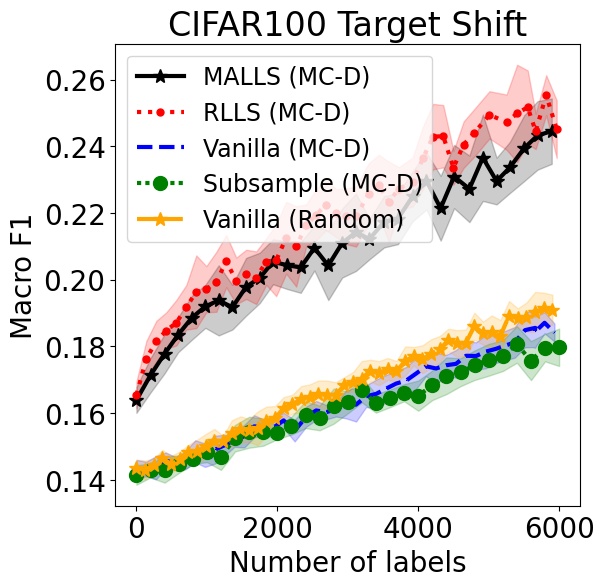}&
      \includegraphics[height=4cm]{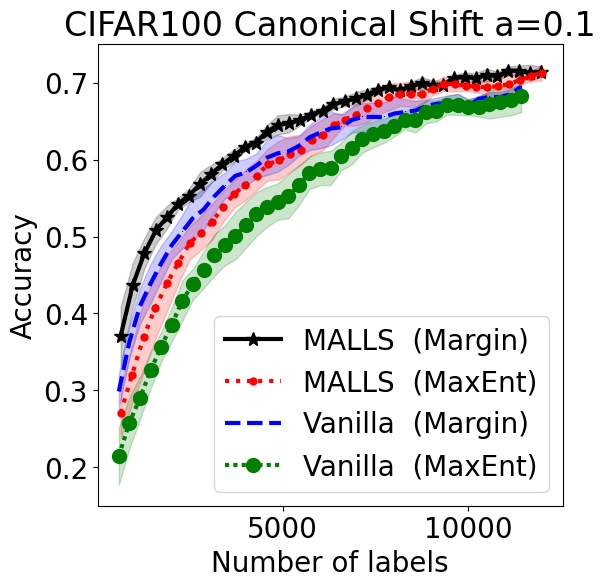}
      &
       \includegraphics[height=4cm]{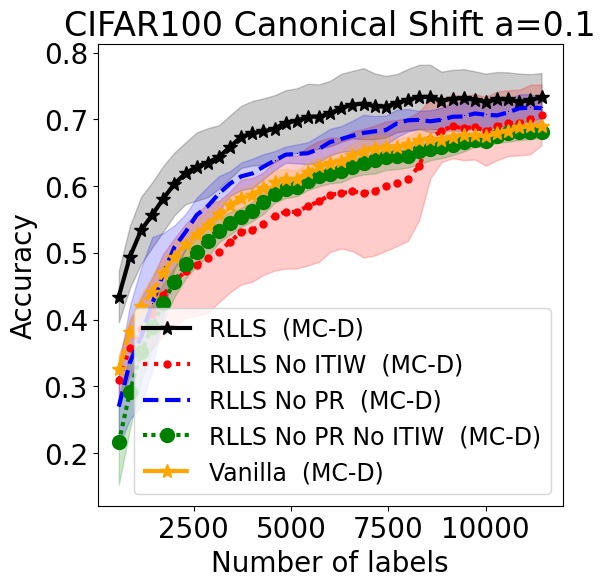}
      \\
            \tiny{(a)} & \tiny{(b)} &\tiny{(c)}   &     \tiny{(d)}  
         
       \end{tabular}
       
         \caption{Average and 95\% confidence intervals on 10 runs of experiments on CIFAR100. Figures (a)(b) depict \textit{general label shift} problems, (c)(d) depict \textit{canonical label shift}  (defined in Preliminaries). (a) Top-5 accuracy under \textit{imbalanced source}, subsampling outperforms importance weighting; (b) Macro F1 under \textit{imbalanced target}, importance weighting outperforms subsampling; (c) \batchedalls{} provides performance gains for multiple popular uncertainty sampling methods; (d) \batchedalls{}'s best practices provides significant gains on performance.
        }
        \label{fig:ablation}
     
    \end{figure*}

\textbf{Experiments}
We demonstrate the \batchedalls{} framework on the ornithology dataset NABirds \cite{van_horn_building_2015} and the benchmark datasets CIFAR10 \& CIFAR100 \cite{krizhevsky_learning_2009}.
Our experiments show \alls{} improves active learning performance under a diverse range of label shift scenarios.

\textbf{Methods}
We evaluate our \batchedalls{} framework on several uncertainty sampling algorithms: (1) Monte Carlo dropout (MC-D) 
\cite{gal_dropout_2016}; (2) maximum entropy sampling (MaxEnt)%
; and (3) maximum margin (Margin). 
We compare against random sampling and active learning without \alls{} (marked \textit{Vanilla}).
In ablation studies, we also compare against only importance weighting or subsampling.
As in Section \ref{sec:alls}, the blackbox hypothesis is obtained by training a model on the warm start data split.

\textbf{Primary Results}
We present our primary results in Figures \ref{fig:big}-\ref{fig:cifar}.
These experiments apply \alls{} to the batch-mode pool-based active learning of Resnet18 models.
The label shift in the NABirds dataset arises from a naturally occurring class imbalance where a dominant class constitutes a near majority of all data \cite{aggarwal_active_2020}.
We adopt this imbalance and assume a uniform test label distribution. 
We artificially induce \textit{canonical label shift} in the CIFAR10 and CIFAR100 experiments by applying \cite{lipton_detecting_2018}'s \textit{Dirichlet Shift} procedure to the unlabeled $\Dun$ and test $\Dte$ datasets. 

In all experiments, \alls{} significantly improves both accuracy and macro F1 scores.
In synthetic shift experiments, \alls{} reduces sample complexity by up to half an order of magnitude.

\textbf{Learning Dynamics of MALLS}
Figure \ref{fig:big}(d) details the learning evolution of \alls{} by depicting a dominant class's accuracy over training time.
The class's accuracy initially declines due to the class's low importance weights, but recovers as the label shift is corrected and the dominant class's importance weight grows.

\textbf{Uncertainty Measures}
Figure \ref{fig:cifar}(c)(d) and \ref{fig:ablation}(c) demonstrates the performance improvements from using \batchedalls{} generalize to several popular uncertainty sampling algorithms.
Importantly, the gains realized by using \batchedalls{} is largely independent of the choice of uncertainty sampling.


\textbf{Imbalanced Source v.s. Imbalanced Target}
Figures \ref{fig:ablation}(a)(b) depicts synthetic \textit{general label shift} problems under \textit{imbalanced source} and \textit{imbalanced target} settings on CIFAR100.
We compare \alls{} against the use of subsampling or importance weighting alone to investigate the trade-off implied by theory.
While Figure \ref{fig:ablation}(a) demonstrates that subsampling accounts for \alls{}'s performance gains under \textit{imbalanced source},  Figure \ref{fig:ablation}(b) demonstrates that importance weighting accounts for \alls{}'s performance gains under \textit{imbalanced target}. 
This corroborates our theoretical analysis.


\textbf{Label Shift Magnitude}
These experiments evaluate MALLS on different magnitudes of label shift, where label shift is induced according to Dirichlet distributions for varying choices of $\alpha$.
Note that shift magnitude is inversely correlated with $\alpha$---smaller $\alpha$ denotes a larger shift.
Figure \ref{fig:alphas} demonstrates that the performance gains introduced by RLLS scale with the magnitude of the label shift.
The results also confirm that the effectiveness of active learning drops under strong label shift.
Plot (a) confirms that even when label shift is negligible, MALLS does not perform significantly worse than vanilla active learning.

\textbf{Best Practices}
Figures \ref{fig:ablation}(d) compares performance when \batchedalls{}'s heuristics of posterior regularization (PR) and iterative reweighting (ITIW) are not used.
Posterior regularization lowers variance (versus importance weighting) and especially improves early-stage performance. Iterative reweighting similarly introduces consistent performance gains. Combining them provides additional gains.

\section{Conclusion}
In this paper, we propose an algorithm for active learning under label shift, \alls{}, with strong label complexity and generalization bounds.
We also introduce a framework, \batchedalls{}, for practitioners to address label shift in real world uncertainty sampling applications.
In many applications that require manually labeling of data, like natural language processing and computer vision, an extension of the techniques we explore in \alls{} may help mitigate bias in the data collection process.
Many problems of theoretical importance---such as cost-sensitive, multi-domain, and Neyman-Pearson settings---share a fundamental connection with the label shift problem.
We believe \alls{} can be extended to provide novel results in these settings as well.


\subsubsection*{Acknowledgements}
Anqi Liu is supported by the PIMCO Postdoctoral Fellowship. Prof. Anandkumar is supported by Bren endowed Chair, faculty awards from Microsoft, Google, and Adobe, Beyond Limits, and LwLL grants.  This work is also supported by funding from Raytheon and NASA TRISH.  

\bibliography{references}

\begin{thebibliography}{}

\bibitem[Aggarwal et~al., 2020]{aggarwal_active_2020}
Aggarwal, U., Popescu, A., and Hudelot, C. (2020).
\newblock Active {Learning} for {Imbalanced} {Datasets}.
\newblock pages 1428--1437.

\bibitem[Azizzadenesheli et~al., 2019]{azizzadenesheli_regularized_2019}
Azizzadenesheli, K., Liu, A., Yang, F., and Anandkumar, A. (2019).
\newblock Regularized {Learning} for {Domain} {Adaptation} under {Label}
  {Shifts}.
\newblock {\em arXiv:1903.09734 [cs, stat]}.
\newblock arXiv: 1903.09734.

\bibitem[Balcan et~al., 2009]{balcan_agnostic_2009}
Balcan, M.-F., Beygelzimer, A., and Langford, J. (2009).
\newblock Agnostic active learning.
\newblock {\em Journal of Computer and System Sciences}, 75(1):78--89.

\bibitem[Ben-David et~al., 2010]{ben-david_theory_2010}
Ben-David, S., Blitzer, J., Crammer, K., Kulesza, A., Pereira, F., and Vaughan,
  J.~W. (2010).
\newblock A theory of learning from different domains.
\newblock {\em Machine Learning}, 79(1-2):151--175.

\bibitem[Ben-David et~al., 2007]{ben-david_analysis_2007}
Ben-David, S., Blitzer, J., Crammer, K., and Pereira, F. (2007).
\newblock Analysis of {Representations} for {Domain} {Adaptation}.
\newblock In Schölkopf, B., Platt, J.~C., and Hoffman, T., editors, {\em
  Advances in {Neural} {Information} {Processing} {Systems} 19}, pages
  137--144. MIT Press.

\bibitem[Beygelzimer et~al., 2009]{beygelzimer_importance_2009}
Beygelzimer, A., Dasgupta, S., and Langford, J. (2009).
\newblock Importance {Weighted} {Active} {Learning}.
\newblock {\em arXiv:0812.4952 [cs]}.
\newblock arXiv: 0812.4952.

\bibitem[Beygelzimer et~al., 2010]{beygelzimer_agnostic_2010}
Beygelzimer, A., Hsu, D., Langford, J., and Zhang, T. (2010).
\newblock Agnostic {Active} {Learning} {Without} {Constraints}.
\newblock {\em arXiv:1006.2588 [cs]}.
\newblock arXiv: 1006.2588.

\bibitem[Chan and Ng, 2007]{chan_domain_2007}
Chan, Y.~S. and Ng, H.~T. (2007).
\newblock Domain {Adaptation} with {Active} {Learning} for {Word} {Sense}
  {Disambiguation}.
\newblock In {\em Proceedings of the 45th {Annual} {Meeting} of the
  {Association} of {Computational} {Linguistics}}, pages 49--56, Prague, Czech
  Republic. Association for Computational Linguistics.

\bibitem[Chattopadhyay et~al., 2013]{chattopadhyay_joint_2013}
Chattopadhyay, R., Fan, W., Davidson, I., Panchanathan, S., and Ye, J. (2013).
\newblock Joint transfer and batch-mode active learning.
\newblock In {\em 30th {International} {Conference} on {Machine} {Learning},
  {ICML} 2013}, pages 1290--1298. International Machine Learning Society
  (IMLS).

\bibitem[Choudhury and Srinivasa, 2020]{choudhury_bayesian_2020}
Choudhury, S. and Srinivasa, S.~S. (2020).
\newblock A {Bayesian} {Active} {Learning} {Approach} to {Adaptive} {Motion}
  {Planning}.
\newblock In {\em Robotics {Research}}, pages 33--40. Springer.

\bibitem[Cortes et~al., 2010]{cortes_learning_2010}
Cortes, C., Mansour, Y., and Mohri, M. (2010).
\newblock Learning {Bounds} for {Importance} {Weighting}.
\newblock In Lafferty, J.~D., Williams, C. K.~I., Shawe-Taylor, J., Zemel,
  R.~S., and Culotta, A., editors, {\em Advances in {Neural} {Information}
  {Processing} {Systems} 23}, pages 442--450. Curran Associates, Inc.

\bibitem[Cortes and Mohri, 2014]{cortes_domain_2014}
Cortes, C. and Mohri, M. (2014).
\newblock Domain adaptation and sample bias correction theory and algorithm for
  regression.
\newblock {\em Theoretical Computer Science}, 519:103--126.
\newblock Publisher: Elsevier.

\bibitem[Deng et~al., 2018]{deng_active_2018}
Deng, C., Liu, X., Li, C., and Tao, D. (2018).
\newblock Active multi-kernel domain adaptation for hyperspectral image
  classification.
\newblock {\em Pattern Recognition}, 77:306--315.
\newblock Publisher: Elsevier.

\bibitem[Gal and Ghahramani, 2016]{gal_dropout_2016}
Gal, Y. and Ghahramani, Z. (2016).
\newblock Dropout as a {Bayesian} {Approximation}: {Representing} {Model}
  {Uncertainty} in {Deep} {Learning}.
\newblock {\em arXiv:1506.02142 [cs, stat]}.
\newblock arXiv: 1506.02142.

\bibitem[Garg et~al., 2020]{garg_unified_2020}
Garg, S., Wu, Y., Balakrishnan, S., and Lipton, Z.~C. (2020).
\newblock A {Unified} {View} of {Label} {Shift} {Estimation}.
\newblock {\em arXiv:2003.07554 [cs, stat]}.
\newblock arXiv: 2003.07554.

\bibitem[Gretton et~al., 2009]{gretton_covariate_2009}
Gretton, A., Smola, A., Huang, J., Schmittfull, M., Borgwardt, K., Schölkopf,
  B., Candela, J., Sugiyama, M., Schwaighofer, A., and Lawrence, N. (2009).
\newblock Covariate {Shift} by {Kernel} {Mean} {Matching}.
\newblock {\em Dataset Shift in Machine Learning, 131-160 (2009)}.

\bibitem[Hanneke, 2007]{hanneke_bound_2007}
Hanneke, S. (2007).
\newblock A bound on the label complexity of agnostic active learning.
\newblock In {\em Proceedings of the 24th international conference on {Machine}
  learning}, {ICML} '07, pages 353--360, Corvalis, Oregon, USA. Association for
  Computing Machinery.

\bibitem[Hanneke, 2011]{hanneke_activized_2011}
Hanneke, S. (2011).
\newblock Activized {Learning}: {Transforming} {Passive} to {Active} with
  {Improved} {Label} {Complexity}.
\newblock {\em arXiv:1108.1766 [cs, math, stat]}.
\newblock arXiv: 1108.1766.

\bibitem[Hanneke, 2014]{hanneke_theory_2014}
Hanneke, S. (2014).
\newblock Theory of {Disagreement}-{Based} {Active} {Learning}.
\newblock {\em Foundations and Trends® in Machine Learning}, 7(2-3):131--309.
\newblock Publisher: Now Publishers, Inc.

\bibitem[Huang and Chen, 2016]{huang_transfer_2016}
Huang, S.-J. and Chen, S. (2016).
\newblock Transfer learning with active queries from source domain.
\newblock In {\em Proceedings of the {Twenty}-{Fifth} {International} {Joint}
  {Conference} on {Artificial} {Intelligence}}, {IJCAI}'16, pages 1592--1598,
  New York, New York, USA. AAAI Press.

\bibitem[Krishnamurthy et~al., 2019]{krishnamurthy_active_2019}
Krishnamurthy, A., Agarwal, A., Huang, T.-K., Daume~III, H., and Langford, J.
  (2019).
\newblock Active {Learning} for {Cost}-{Sensitive} {Classification}.
\newblock {\em arXiv:1703.01014 [cs, stat]}.
\newblock arXiv: 1703.01014.

\bibitem[Krizhevsky, 2009]{krizhevsky_learning_2009}
Krizhevsky, A. (2009).
\newblock Learning {Multiple} {Layers} of {Features} from {Tiny} {Images}.

\bibitem[Lin et~al., 2018]{lin_active_2018}
Lin, C.~H., Mausam, M., and Weld, D.~S. (2018).
\newblock Active {Learning} with {Unbalanced} {Classes} and
  {Example}-{Generation} {Queries}.
\newblock In {\em Sixth {AAAI} {Conference} on {Human} {Computation} and
  {Crowdsourcing}}.

\bibitem[Lipton et~al., 2018]{lipton_detecting_2018}
Lipton, Z.~C., Wang, Y.-X., and Smola, A. (2018).
\newblock Detecting and {Correcting} for {Label} {Shift} with {Black} {Box}
  {Predictors}.

\bibitem[Matasci et~al., 2012]{matasci_svm-based_2012}
Matasci, G., Tuia, D., and Kanevski, M. (2012).
\newblock {SVM}-based boosting of active learning strategies for efficient
  domain adaptation.
\newblock {\em IEEE Journal of Selected Topics in Applied Earth Observations
  and Remote Sensing}, 5(5):1335--1343.
\newblock Publisher: IEEE.

\bibitem[Park, 2011]{park_improved_2011}
Park, W.~J. (2011).
\newblock An {Improved} {Active} {Learning} in {Unbalanced} {Data}
  {Classification}.
\newblock In Lee, C., Seigneur, J.-M., Park, J.~J., and Wagner, R.~R., editors,
  {\em Secure and {Trust} {Computing}, {Data} {Management}, and
  {Applications}}, Communications in {Computer} and {Information} {Science},
  pages 84--93, Berlin, Heidelberg. Springer.

\bibitem[Rai et~al., 2010]{rai_domain_2010}
Rai, P., Saha, A., Daumé, H., and Venkatasubramanian, S. (2010).
\newblock Domain {Adaptation} meets {Active} {Learning}.
\newblock In {\em Proceedings of the {NAACL} {HLT} 2010 {Workshop} on {Active}
  {Learning} for {Natural} {Language} {Processing}}, pages 27--32, Los Angeles,
  California. Association for Computational Linguistics.

\bibitem[Saerens et~al., 2002]{saerens_adjusting_2002}
Saerens, M., Latinne, P., and Decaestecker, C. (2002).
\newblock Adjusting the outputs of a classifier to new a priori probabilities:
  a simple procedure.
\newblock {\em Neural computation}, 14(1):21--41.
\newblock Publisher: MIT Press.

\bibitem[Saha et~al., 2011]{saha_active_2011}
Saha, A., Rai, P., Daumé, H., Venkatasubramanian, S., and DuVall, S.~L.
  (2011).
\newblock Active {Supervised} {Domain} {Adaptation}.
\newblock In Gunopulos, D., Hofmann, T., Malerba, D., and Vazirgiannis, M.,
  editors, {\em Machine {Learning} and {Knowledge} {Discovery} in {Databases}},
  Lecture {Notes} in {Computer} {Science}, pages 97--112, Berlin, Heidelberg.
  Springer.

\bibitem[Shen et~al., 2018]{shen_deep_2018}
Shen, Y., Yun, H., Lipton, Z.~C., Kronrod, Y., and Anandkumar, A. (2018).
\newblock Deep {Active} {Learning} for {Named} {Entity} {Recognition}.
\newblock {\em arXiv:1707.05928 [cs]}.
\newblock arXiv: 1707.05928.

\bibitem[Shimodaira, 2000]{shimodaira_improving_2000}
Shimodaira, H. (2000).
\newblock Improving predictive inference under covariate shift by weighting the
  log-likelihood function.
\newblock {\em Journal of statistical planning and inference}, 90(2):227--244.
\newblock Publisher: Elsevier.

\bibitem[Su et~al., 2019]{su_active_2019}
Su, J.-C., Tsai, Y.-H., Sohn, K., Liu, B., Maji, S., and Chandraker, M. (2019).
\newblock Active {Adversarial} {Domain} {Adaptation}.
\newblock {\em arXiv:1904.07848 [cs]}.
\newblock arXiv: 1904.07848 version: 1.

\bibitem[Sugiyama et~al., 2007]{sugiyama_covariate_2007}
Sugiyama, M., Krauledat, M., and Müller, K.-R. (2007).
\newblock Covariate {Shift} {Adaptation} by {Importance} {Weighted} {Cross}
  {Validation}.
\newblock {\em The Journal of Machine Learning Research}, 8:985--1005.

\bibitem[Sugiyama et~al., 2012]{sugiyama_density_2012}
Sugiyama, M., Suzuki, T., and Kanamori, T. (2012).
\newblock {\em Density ratio estimation in machine learning}.
\newblock Cambridge University Press.

\bibitem[Tsuboi et~al., 2009]{tsuboi_direct_2009}
Tsuboi, Y., Kashima, H., Hido, S., Bickel, S., and Sugiyama, M. (2009).
\newblock Direct density ratio estimation for large-scale covariate shift
  adaptation.
\newblock {\em Journal of Information Processing}, 17:138--155.
\newblock Publisher: Information Processing Society of Japan.

\bibitem[Van~Horn et~al., 2015]{van_horn_building_2015}
Van~Horn, G., Branson, S., Farrell, R., Haber, S., Barry, J., Ipeirotis, P.,
  Perona, P., and Belongie, S. (2015).
\newblock Building a bird recognition app and large scale dataset with citizen
  scientists: {The} fine print in fine-grained dataset collection.
\newblock In {\em Proceedings of the {IEEE} {Conference} on {Computer} {Vision}
  and {Pattern} {Recognition}}, pages 595--604.

\bibitem[Yamada et~al., 2011]{yamada_relative_2011}
Yamada, M., Suzuki, T., Kanamori, T., Hachiya, H., and Sugiyama, M. (2011).
\newblock Relative density-ratio estimation for robust distribution comparison.
\newblock In {\em Advances in neural information processing systems}, pages
  594--602.

\bibitem[Yan et~al., 2018]{yan_active_2018}
Yan, S., Chaudhuri, K., and Javidi, T. (2018).
\newblock Active {Learning} with {Logged} {Data}.
\newblock {\em arXiv:1802.09069 [cs, stat]}.
\newblock arXiv: 1802.09069.

\bibitem[Yang et~al., 2020]{yang_towards_2020}
Yang, K., Qinami, K., Fei-Fei, L., Deng, J., and Russakovsky, O. (2020).
\newblock Towards fairer datasets: {Filtering} and balancing the distribution
  of the people subtree in the imagenet hierarchy.
\newblock In {\em Proceedings of the 2020 {Conference} on {Fairness},
  {Accountability}, and {Transparency}}, pages 547--558.

\bibitem[Yang and Ma, 2010]{yang_ensemble-based_2010}
Yang, Y. and Ma, G. (2010).
\newblock Ensemble-based active learning for class imbalance problem.
\newblock {\em Journal of Biomedical Science and Engineering},
  3(10):1022--1029.
\newblock Number: 10 Publisher: Scientific Research Publishing.

\bibitem[Yang et~al., 2015]{yang_multi-class_2015}
Yang, Y., Ma, Z., Nie, F., Chang, X., and Hauptmann, A.~G. (2015).
\newblock Multi-class active learning by uncertainty sampling with diversity
  maximization.
\newblock {\em International Journal of Computer Vision}, 113(2):113--127.
\newblock Publisher: Springer.

\bibitem[Zhang, 2005]{zhang_data_2005}
Zhang, T. (2005).
\newblock Data {Dependent} {Concentration} {Bounds} for {Sequential}
  {Prediction} {Algorithms}.
\newblock pages 173--187.

\end{thebibliography}
\bibliographystyle{apalike}

\newpage
\onecolumn
\aistatstitle{Supplementary Materials}

\section{Proofs}
\setcounter{theorem}{3}

\subsection{Proof of Theorem 1}

We formalize the violation of label shift assumptions resulting from subsampling as label shift drift \cite{azizzadenesheli_regularized_2019}.
\begin{lemma}
The drift from label shift is bounded by:
\begin{align}
    \abs{ 1 - \expc{X, Y \sim \Pte} \brcksq{ \frac{\Pmed(x | y)}{\Pte(x | y)} } }
    \leq \norm{\rsm}_{\infty} \err(h_0, \rsm)
\end{align}
\label{lemma:drift}
\end{lemma}
\begin{proof}
The drift is equivalent to expected importance weights,
\begin{align}
    \abs{1 - \expc{X, Y \sim \Pte} \brcksq{ \frac{\Pmed(x| y)}{\Pte(x |y)} }}
    & = \abs{1 - \int_{X, Y} \Pmed(x | y)\Pte(y)} \nonumber \\
    & = \abs{1 - \int_{X, Y}  \Pmed(x, y)\frac{\Pte(y)}{\Pmed(y)} } \nonumber \\
    & = \abs{1 - \expc{X, Y \sim \Pmed} \brcksq{\frac{\Pte(y)}{\Pmed(y)} }}
    \end{align}
Drift can therefore be estimated in practice by randomly labeling subsampled points and measuring the average importance weight value.
We can further expand the value of drift as:
\begin{align}
\abs{1 - \expc{X, Y \sim \Pmed} \brcksq{\frac{\Pte(y)}{\Pmed(y)} }}
    & = \abs{1 - \int_{X, Y} C \Ps(x, y) \Pss(h_0(x))\frac{\Pte(y)}{\Pmed(y)} } \nonumber \\
    & = \abs{1 - C \expc{X, Y \sim \Ps} \brcksq{ \Pss(h_0(x))\frac{\Pte(y)}{\Pmed(y)} }} \nonumber \\
    & = \abs{1 - C \expc{X, Y \sim \Ps} \brcksq{ \Pss(y)\frac{\Pte(y)}{\Pmed(y)} }} + \abs{C \expc{X, Y \sim \Ps} \brcksq{ \brck{\Pss(h_0(x)) - \Pss(y)} \frac{\Pte(y)}{\Pmed(y)} }} \nonumber \\
    & = \abs{1 - \sum_{Y} \brcksq{ \Pmed^*(y)\frac{\Pte(y)}{\Pmed(y)} }} + \abs{C \expc{X, Y \sim \Ps} \brcksq{ \brck{\Pss(h_0(x)) - \Pss(y)} \frac{\Pte(y)}{\Pmed(y)} }}
\end{align}
where $C$ is a constant where $\Pss = \frac{1}{C} \frac{\Pmed}{\Ps}$ and $\Pmed^*$ denotes the target medial distribution.
The second term corresponds to a weighted L1 error on $\Psrc$.
\begin{align}
    \abs{C \expc{X, Y \sim \Ps} \brcksq{ \brck{\Pss(h_0(x)) - \Pss(y)} \frac{\Pte(y)}{\Pmed(y)} }}
    & \leq \norm{\rsm}_{\infty}  \expc{X, Y \sim \Ps} \brcksq{ \abs{ \mathbbm{1}[h_0(x) \neq y]} \frac{\Pte(y)}{\Pmed(y)} } \nonumber  \\
    & = \norm{\rsm}_{\infty} \err(h_0, \rsm)
\end{align}
where $\err(h_0, r)$ denotes the importance weighted 0/1-error of a blackbox predictor $h_0$ on $Ps$.
As the first term is thus dominated, we have that drift is bounded by the accuracy of the blackbox hypothesis.
\end{proof}

Plugging Lemma \ref{lemma:drift} into Theorem 2 in \cite{azizzadenesheli_regularized_2019} yields a generalization of Theorem 1 where the number of unlabeled datapoints from the test distribution is $n'$.
\begin{theorem}
With probability $1 - \delta$, for all $n \geq 1$:
\begin{align}
| \Delta |
& \leq \mathcal{O} \left(
    \frac{2 }{\sigma_{\min}} \left(
        \norm{\thetamt}_2
        \sqrt{
            \frac{\log \brck{\frac{n k}{\delta}}}{n}
        }
        +
        \sqrt{
            \frac{\log \brck{\frac{ n}{\delta}}}{n}
        }
        + 
        \sqrt{
            \frac{\log \brck{\frac{ n}{\delta}}}{n'}
        }
        + 
        \norm{\thetasm}_{\infty}
        \text{err}(h_0, \rmt) 
    \right)
\right)
\end{align}
where $\sigma_{\min}$ denotes the smallest singular value of the confusion matrix and $\err(h_0, r)$ denotes the importance weighted $0/1$-error of a blackbox predictor $h_0$ on $\Psrc$.
\end{theorem}
Theorem 1 follows by setting $n' \rightarrow \infty$.

\subsection{Theorem 2 and Theorem 3 Proofs}
We will prove Theorem 2 and Theorem 3 for the general case where the number of unlabeled datapoints from the test distribution is $n'$.
For the case depicted in the main paper, set $n' \rightarrow \infty$.

First, we review the IWAL-CAL active learning algorithm \cite{beygelzimer_agnostic_2010}.
Let $\text{err}_{S_i}(h) \rightarrow [0, 1]$ denote the error of hypothesis $h \in H$ as estimated on $S_i$ while $\text{err}_{\Pte}(h)$ denote the expected error of $h$ on $\Pte$.
We next define,
\begin{align}
    h^* & := \text{argmin}_{h \in H} \text{err}_{\Pte}(h), \nonumber  \\
    h_k & := \text{argmin}_{h \in H} \text{err}_{S_{k-1}}(h), \nonumber  \\
    h'_k & := \text{argmin} \{ \text{err}_{S_{k-1}}(h) \mid h \in H \wedge h(\textbf{D}_{\text{unlab}}^{(k)}) \neq h_k(\textbf{D}_{\text{unlab}}^{(k)})\} \nonumber \\
    G_k & := \text{err}_{S_{k-1}}(h'_k) - \text{err}_{S_{k-1}}(h_k) \nonumber 
\end{align}
IWAL-CAL employs a sampling probability $P_t = \min \{1, s\}$ for the $s \in (0, 1)$ which solves the equation,
\begin{align}
    G_t = \left( \frac{c_1}{\sqrt{s}} - c_1 + 1 \right) \sqrt{\frac{C_0 \log t}{t - 1}}\notag
          + \left(\frac{c_2}{s} - c_2 + 1 \right) \frac{C_0 \log t}{t - 1}  
\end{align}
where $C_0$ is a constant bounded in Theorem 2 and $c_1 := 5 + 2 \sqrt{2}, c_2 := 5$.

The most involved step in deriving generalization and sample complexity bounds for MALLS is bounding the deviation of empirical risk estimates. This is done through the following theorem.
\begin{theorem}
\label{thm:dev}
Let $Z_i := (X_i, Y_i, Q_i)$ be our source data set, where $Q_i$ is the indicator function on whether $(X_i, Y_i)$ is sampled as labeled data. The following holds for all $n \geq 1$ and all $h \in \mathcal{H}$ with probability $1 - \delta$:
\begin{align}
\label{eq:dev}
& \left | err(h, Z_{1:n}) - err(h^*, Z_{1:n}) - err(h) + err(h^*) \right| \nonumber  \\
& \leq \mathcal{O} 
     \left(
    (2 + \norm{\theta}_2) \sqrt{\frac{\varepsilon_n}{P_{\min, n}(h)}} + \frac{\varepsilon_n}{P_{\min, n}(h)}
    + \frac{
        2 d_{\infty} (\Pte, \Psrc) \log (\frac{2 n |H|}{\delta})
    }{
        3 n
    }
    + \sqrt{\frac{  
        2 d_2 (\Pte, \Psrc) \log (\frac{2 n |H|}{\delta})
    }{
        n
    }}
    \right.
     \\
    & \left.
    + \norm{\rsm}_{\infty} \err(h_0, \rsm)
    + \frac{2 }{\sigma_{\min}} \left(
        \norm{\thetamt}_2
        \sqrt{
            \frac{\log \brck{\frac{n k}{\delta}}}{\lambda n}
        }
        +
        \sqrt{
            \frac{\log \brck{\frac{n}{\delta}}}{\lambda n}
        }
        + 
        \sqrt{
            \frac{\log \brck{\frac{n}{\delta}}}{n'}
        }
        + 
        \norm{\thetasm}_{\infty}
        \text{err}(h_0, \rmt) 
    \right)
\right)\nonumber 
\end{align}
where $\varepsilon_n := \frac{16 \log(2 ( 2 + n \log_2 n) n (n+1) |H| / \delta)}{n}$.
\end{theorem}

For reading convenience, we set $\Psrc := \Pun$.
This deviation bound will plug in to IWAL-CAL for generalization and sample complexity bounds.
In the remainder of this appendix section, we detail our proof of Theorem \ref{thm:dev}.
We proceed by expressing Theorem \ref{thm:dev} in a more general form with a bounded function $f: X \times Y \rightarrow [-1, 1]$ which will eventually represent $\text{err}(h) - \text{err}(h^*)$.

We borrow notation for the terms $W, Q$ from \cite{beygelzimer_agnostic_2010}, where $Q_i$ is an indicator random variable indicating whether the $i$th datapoint is labeled and $W := Q_i \tilde{Q}_i \rmt^{(i)} f(x_i, y_i)$.
We use the shorthand $r^{(i)}$ for the $y_i$th component of importance weight $r$. 
Similarly, the indicator random variable $\tilde{Q}_i$ indicates whether the $i$th data sample is retained by the subsampler.
The expectation $\expc{i}[W]$ is taken over the randomness of $Q$ and $\tilde{Q}$. 
We also borrow \cite{azizzadenesheli_regularized_2019}'s label shift notation and define $k$ as the size of the output space (finite) and denote estimated importance weights with hats, e.g. $\hat{r}$.
We also introduce a variant of $W$ using estimated importance weights $r$:  $\hat{W} := Q_i \tilde{Q}_i \hatrmt^{(i)} f(x_i, y_i)$.
Finally, we follow \cite{cortes_learning_2010} and use $d_\alpha(P || P')$ to denote $2^{D_\alpha(P || P')}$ where $D_\alpha(P || P') := \log (\frac{P_i}{P'_i})$ is the Renyi divergence of distributions $P$ and $P'$.

We seek to bound with high probability,
\begin{align}
 \abs{\Delta} := \abs{\frac{1}{n} \left(\sum_{i=1}^{n} 
 \hat{W}_i \right) - \mathbb{E}_{x, y \sim \Pt} [f(x, y)]}  \leq |\Delta_1| + |\Delta_2| + |\Delta_3| + \abs{\Delta_4} 
\end{align}
where,
\begin{align}
    \Delta_1 & := \expc{x, y \sim \Pt} [f(x, y)] - \expc{x, y \sim \Ps} [W_i], \nonumber 
    \\
    \Delta_2 & := \expc{x, y \sim \Ps} [W_i] - \frac{1}{n} \sum_{i=1}^{n} \expc{i} \brcksq{W_i}, 
    \nonumber  \\
    \Delta_3 & := \frac{1}{n} \sum_{i=1}^{n} \expc{i}\brcksq{W_i} - \expc{i}\brcksq{\hat{W}_i} 
    \nonumber \\
    \Delta_4 & := \frac{1}{n} \sum_{i=1}^{n} \expc{i} [\hat{W}_i] - \hat{W}_i \nonumber 
\end{align}
$\Delta_1$ corresponds to the drift from label shift introduced by subsampling,
$\Delta_2$ to finite-sample variance.
and $\Delta_3$ to label shift estimation errors.
The final $\Delta_4$ corresponds to the variance from randomly sampling.

We bound $\Delta_4$ using a Martingale technique from \cite{zhang_data_2005} also adopted by \cite{beygelzimer_agnostic_2010}.
We take Lemmas 1, 2 from \cite{zhang_data_2005} as given.
We now proceed in a fashion similar to the proof of Theorem 1 from \cite{beygelzimer_agnostic_2010}.
We begin with a generalization of Lemma 6 in \cite{beygelzimer_agnostic_2010}.
\begin{lemma}
\label{lemma:d4l6}
If $0 < \lambda < 3 \frac{P_i}{\hatrmt^{(i)}}$, then
\begin{align}
    \log \expc{i}{}[\exp( \lambda ( \hat{W}_i - \expc{i}{}[\hat{W}_i]))]
    \leq \frac{\hat{r}_i \hatrmt^{(i)} \lambda^2}{2 P_i(1 - \frac{\hatrmt^{(i)} \lambda}{3 P_i})}
\end{align}
where $\hat{r}_i := \hatrmt^{(i)} \expc{i}{}[\tilde{Q}_i]$.
If $\expc{i}{}[\hat{W}_i] = 0$ then
\begin{align}
    \log \expc{i}{}[\exp(\lambda(\hat{W}_i - \expc{i}{}[\hat{W}_i]))] = 0
\end{align}
\end{lemma}
\begin{proof}
First, we bound the range and variance of $\hat{W_i}$.
The range is trivial
\begin{align}
    |\hat{W_i}| \leq \left| \frac{Q_i \tilde{Q}_i \hatrmt^{(i)}}{P_i} \right| \leq \frac{\hatrmt^{(i)}}{P_i} 
\end{align}
Since subsampling and importance weighting ideally corrects underlying label shift, we can simplify the variance as,
\begin{align}
    \expc{i}{}[(\hat{W}_i - \expc{i}{}[\hat{W}_i])^2]
    & \leq \frac{\hat{r}_i \hatrmt^{(i)}}{P_i} f(x_i, y_i)^2 - 2 \hat{r}_i^2 f(x_i, y_i)^2 + \hat{r}_i^2 f(x_i, y_i)^2
    \leq \frac{\hat{r}_i \hatrmt^{(i)}}{P_i}  
\end{align}

Following \cite{beygelzimer_agnostic_2010}, we choose a function $g(x) := (\exp(x) - x - 1)/x^2$ for $x \neq 0$ so that $\exp(x) = 1 + x + x^2 g (x)$ holds. Note that $g(x)$ is non-decreasing.
Thus,
\begin{align}
    \expc{i}{}[\exp(\lambda( \hat{W}_i - \expc{i}{}[ \hat{W}_i]))] 
   &  = \expc{i}{}[1 + \lambda( \hat{W}_i - \expc{i}{}[ \hat{W}_i ])
                   + \lambda^2 (  \hat{W}_i - \expc{i}{}[ \hat{W}_i] )^2
                   g(\lambda( \hat{W}_i - \expc{i}{}[ \hat{W}_i]))] \nonumber \\
    & = 1 + \lambda^2 \expc{i}{}[( \hat{W}_i - \expc{i}{}[ \hat{W}_i])^2 g(\lambda( \hat{W}_i - \expc{i}{}[ \hat{W}_i]))] \nonumber \\
    & \leq 1 + \lambda^2 \expc{i}{}[( \hat{W}_i - \expc{i}{}[ \hat{W}_i])^2 g(\lambda \hatrmt^{(i)} / P_i)] \nonumber \\
    & = 1 + \lambda^2 \expc{i}{}[( \hat{W}_i - \expc{i}{}[ \hat{W}_i])^2] g(\lambda \hatrmt^{(i)} / P_i) \nonumber \\
    & \leq 1 + \frac{\lambda^2 \hat{r}_i \hatrmt^{(i)}}{P_i} g(\frac{\hatrmt^{(i)}   \lambda}{P_i})
\end{align}
where the first inequality follows from our range bound and the second follows from our variance bound.
The first claim then follows from the definition of $g(x)$ and the facts that $\exp(x) - x - 1 \leq x^2/(2(1-x/3))$ for $0 \leq x < 3$ and $\log(1+x) \leq x$.
The second claim follows from definition of $\hat{W}_i$ and the fact that $\expc{i}{}[\hat{W}_i] = \hat{r} f(X_i, Y_i)$.
\end{proof}

The following lemma is an analogue of Lemma 7 in \cite{beygelzimer_agnostic_2010}.
\begin{lemma}
Pick any $t \geq 0, p_{\min} > 0$ and let $E$ be the joint event
\begin{align}
    \frac{1}{n} \sum_{i=1}^n \hat{W}_i - \sum_{i=1}^n \expc{i}{}[\hat{W}_i]
    \geq (1 + M) \sqrt{\frac{t}{2n p_{\min}}} + \frac{t}{3n p_{\min}} 
    \nonumber \\
    \text{ and }
    \min \{ \frac{P_i}{\hatrmt^{(i)}} : 1 \leq i \leq n \wedge \expc{i}{}[W_i] \neq 0 \} \geq p_{\min}
\end{align}
Then $\Pr(E) \leq e^{-t}$ where $M := \frac{1}{n} \sum_{i=1}^n \hat{r}_i$.
\end{lemma}
\begin{proof}
We follow \cite{beygelzimer_agnostic_2010} and let
\begin{align}
    \lambda := 3 p_{\min} \frac{\sqrt{\frac{2t}{9 n p_{\min}}}} {1 + \sqrt{\frac{2t}{9 n p_{\min}}}}
\end{align}
Note that $0 < \lambda < 3p_{\min}$.
By Lemma \ref{lemma:d4l6}, we know that if $\min \{\frac{P_i}{\hatrmt^{(i)}} : 1 \leq i \leq n \wedge \expc{i}{}[\hat{W}_i] \neq 0 \} \geq p_{\min}$ then
\begin{align}
    \frac{1}{n \lambda} \sum_{i=1}^n \log \expc{i}{}[
        \exp( \lambda ( W_i - \expc{i}{}[W_i] ) )
    ]
    \leq \frac{1}{n} \sum_{i=1}^n \frac{\hat{r}_i \hatrmt^{(i)} \lambda}{2 P_i (1 - \frac{\hatrmt^{(i)} \lambda}{3 P_i}) } \leq M \sqrt{\frac{t}{2 n p_{\min}}}
\end{align}
and
\begin{align}
    \frac{t}{n \lambda}
    = \sqrt{\frac{t}{2 n p_{\min}}} + \frac{t}{3 n p_{\min}}
\end{align}
Let $E'$ be the event that
\begin{align}
    \frac{1}{n} \sum_{i=1}^n (\hat{W}_i - \expc{i}{}[\hat{W}_i])
    - \frac{1}{n \lambda} \sum_{i=1}^n \log \expc{i}{}[\exp( \lambda(\hat{W} - \expc{i}{}[\hat{W}]))]
    \geq \frac{t}{n \lambda}
\end{align}
and let $E''$ be the event $\min \{\frac{P_i}{\hatrmt^{(i)}} : 1 \leq i \leq n \wedge \expc{i}{}[\hat{W}_i] \neq 0\} \geq p_{\min}$. Together, the above two equations imply $E \subseteq E' \bigcap E''$.
By \cite{zhang_data_2005}'s lemmas 1 and 2, $\Pr(E) \leq \Pr(E' \bigcap E'') \leq Pr(E') \leq e^{-t}$.
\end{proof}

The following is an immediate consequence of the previous lemma.
\begin{lemma}
\label{lemma:midstocha}
Pick any $t \geq 0$ and $n \geq 1$.
Assume $1 \leq \frac{\hatrmt^{(i)}}{P_i} \leq r_{\max}$ for all $1 \leq i \leq n$, and let $R_n := \max \{ \frac{\hatrmt^{(i)}}{P_i} : 1 \leq i \leq n \wedge \expc{i}{}[\hat{W}] \neq 0 \} \bigcup \{1\} $.
We have
\begin{align}
    \Pr \left( \left | \frac{1}{n} \sum_{i=1}^n \hat{W}_i - \frac{1}{n} \sum_{i=1}^n \expc{i}{}[\hat{W}_i] \right | \geq (1 + M) \sqrt{\frac{R_n t}{2n}} + \frac{R_n t}{3n} \right) \leq 2(2 + \log_2 r_{\max}) e^{-t/2}
\end{align}
\end{lemma}
\begin{proof}
This proof follows identically to \cite{beygelzimer_agnostic_2010}'s lemma 8.
\end{proof}

We can finally bound $\Delta_4$ by bounding the remaining free quantity $M$.
\begin{lemma}
\label{lemma:maind4}
With probability at least $1 - \delta$, the following holds over all $n \geq 1$ and $h \in H$:
\begin{align}
    \left| \Delta_4 \right|
    \leq (2 + \norm{\hat{\theta}}_2) \sqrt{\frac{\varepsilon_n}{P_{\min, n}(h)}} + \frac{\varepsilon_n}{P_{\min, n}(h)}
\end{align}
where $\varepsilon_n := \frac{16 \log(2 ( 2 + n \log_2 n) n (n+1) |H| / \delta)}{n}$
and $P_{\min, n}(h) = \min \{P_i:1 \leq i \leq n \wedge h(X_i) \neq h^*(X_i) \} \bigcup \{1\}$. 
\end{lemma}
\begin{proof}
We define the $k$-sized vector $\tilde{\ell}(j) = \frac{1}{n} \sum_{i=1}^n \mathds{1}_{y_i = j} \hat{\theta}(j)$.
Here, $v(j)$ is an abuse of notation and denotes the $j$th element of a vector $v$.
Note that we can write $M$ by instead summing over labels, $M = \frac{1}{n} \sum_{i=1}^n \hat{\theta}_i = \sum_{j=1}^k \tilde{\ell}(j)$.
Applying the Cauchy-Schwarz inequality, we have that $\frac{1}{n} \sum_{i=1}^n \hat{\theta}_i \leq \frac{1}{n} \norm{\hat{\theta}}_2 \norm{\dot{\ell}}_2$ where $\dot{\ell}(j)$ is another $k$-sized vector where $\dot{\ell}(j) := \sum_{i=1}^n \mathds{1}_{y_i = j}$.
Since $\norm{\dot{\ell}}_2 \leq n$, we have that $M \leq 1 + \norm{\hat{\theta}}_2$.
The rest of the claim follows by lemma \ref{lemma:midstocha} and a union bound over hypotheses and datapoints.
\end{proof}

The term $\Delta_1$ is be bounded with Theorem 1.
We now bound $\Delta_2$.
This is a simple generalization bound of an importance weighted estimate of $f$.

\begin{lemma}
For any $\delta > 0$, with probability at least $1 - \delta$, then for all $n \geq 1$, $h \in H$:
\begin{align}
    \left | \Delta_2 \right | \leq
    \frac{
        2 d_{\infty} (\Pte, \Psrc) \log (\frac{2 n |H|}{\delta})
    }{
        3n
    }
    + \sqrt{\frac{  
        2 d_2 (\Pte, \Psrc) \log (\frac{2 n |H|}{\delta})
    }{
        n
    }}
\end{align}
\end{lemma}
\begin{proof}
This inequality is a direct application of Theorem 2 from \cite{cortes_learning_2010}.
\end{proof}

The following lemma bounds the remaining term $\Delta_1$.
\begin{lemma}
For all $n \geq 1, h \in H$:
\begin{align}
    \abs{\Delta_1} \leq \norm{\rsm}_{\infty} \err(h_0, \rsm)
\end{align}
\end{lemma}
\begin{proof}
This inequality follows from our Lemma \ref{lemma:drift} and \cite{azizzadenesheli_regularized_2019}'s Theorem 2.
\end{proof}

Theorem \ref{thm:dev} follows by applying a triangle inequality over $\Delta_1, \Delta_2, \Delta_3, \Delta_4$.
If a warm start of $m$ datapoints sampled from $\Pwarm$ is used, the deviation bound is instead:
\begin{align}
\label{eq:warmdev}
& \left | err(h, Z_{1:n}) - err(h^*, Z_{1:n}) - err(h) + err(h^*) \right| \nonumber  \\
& \leq \mathcal{O} 
     \left(
    (2 + \frac{n \norm{\thetaut}_2 + m \norm{\thetawt}_2}{n+m}) \sqrt{\frac{\varepsilon_n}{P_{\min, n}(h)}} + \frac{\varepsilon_n}{P_{\min, n}(h)}
    + \frac{
        2 d_{\infty} (\Pte, \Psrc) \log (\frac{2 n |H|}{\delta})
    }{
        3 ( n + m )
    }
    \right.
    \nonumber \\
    & \left.
    + \sqrt{\frac{  
        2 d_2 (\Pte, \Psrc) \log (\frac{2 n |H|}{\delta})
    }{
        n + m
    }}
    + \frac{n}{n+m} \norm{\rsm}_{\infty} \err(h_0, \rsm)
    \right.
    \nonumber \\
    & \left.
    + \frac{n}{\sigma_{\min}} \left(
        \norm{\thetamt}_2
        \sqrt{
            \frac{\log \brck{\frac{n k}{\delta}}}{\lambda n}
        }
        +
        \sqrt{
            \frac{\log \brck{\frac{n}{\delta}}}{\lambda n}
        }
        + 
        \sqrt{
            \frac{\log \brck{\frac{n}{\delta}}}{n'}
        }
        + 
        \norm{\thetasm}_{\infty}
        \text{err}(h_0, \rmt) 
    \right)
\right)\nonumber 
\end{align}
The only change is that variance and subsampling terms are scaled by $\frac{n}{n+m}$, both of which disappear in the limit where $n >> m$. 
For the remainder of this proof, we continue to set $m = 0$.


Theorem 2 follows by replacing the deviation bound in \cite{beygelzimer_agnostic_2010}'s Theorem 2 with our Theorem \ref{thm:dev}.
Theorem 3 similarly follows from \cite{beygelzimer_agnostic_2010}'s Theorem 3 but with two additions.
First, $\lambda n$ datapoints are sampled for label shift estimation.
Second, the number of datapoints which are either accepted or rejected by the active learning algorithm can be much smaller than the number of datapoints sampled from $\Ps$ due to subsampling.
We can determine this proportion with an upper-tail Chernoff bound.
\begin{lemma}
When $\epsilon < 2^{(-2e-1)/\norm{\rsm}_{\infty}}$, given $n$ datapoints from $\Ps$, subsampling will yield $\textbf{n}$ where,
\begin{align}
    \Pr\brck{\textbf{n} \geq \frac{n}{\norm{\rsm}_{\infty}} + \log_2 \brck{\frac{1}{\epsilon}}} \leq \epsilon
\end{align}
\end{lemma}
\begin{proof}
The number of subsampled datapoints is sum of independent Bernoulli trials with mean $\mu$,
\begin{align}
    \mu = \expc{y \sim \Ps} \brcksq{\Pss(y)} =  \expc{y \sim \Ps} \brcksq{C \frac{\Pmed(y)}{\Ps(y)}}
    = \expc{y \sim \Pmed} \brcksq{C}
    = C
\end{align}
where $C$ is a constant such that $C \frac{\Pmed(y)}{\Ps(y)} \leq 1$ for all labels $y$.
Thus, $\mu = C \leq 1 / \norm{\rsm}_{\infty}$.
\end{proof}

\newpage
\section{Supplementary Experiments}
\setcounter{figure}{7}

\subsection{NABirds Regional Species Experiment}
We conduct an additional experiment on the NABirds dataset using the grandchild level of the class label hierarchy, which results in 228 classes in total.
These classes correspond to individual species and present a significantly larger output space than considered in Figure 6.
For realism, we retain the original training distribution in the dataset as the source distribution; sampling I.I.D. from the original split in the experiment.
To simulate a scenario where a bird species classifier is adapted to a new region with new bird frequencies, we induce an imbalance in the target distribution to render certain birds more common than others.
Table \ref{tab:nabirds} demonstrates the average accuracy of our framework at different label budgets.
We observe consistent gains in accuracy at different label budgets.

\begin{table}[h]
    \centering
    \begin{center}
     \begin{tabular}{||c c c c ||} 
     \hline
     Strategy & Acc (854 Labels) & Acc (1708) & Acc (3416) \\ [0.5ex] 
     \hline\hline
     MALLS (MC-D) & \textbf{0.51} &  \textbf{0.53} &  \textbf{0.56} \\
     \hline
     Vanilla (MC-D) & 0.46 & 0.48 & 0.50 \\
     \hline
     Random & 0.38 & 0.40 & 0.42 \\
     \hline
    \end{tabular}
    \end{center}
    \caption{NABirds (species) Experiment Average Accuracy}
    \label{tab:nabirds}
\end{table}

\subsection{Change in distribution}
To further analyze the learning behavior of MALLS, we can analyze the label distribution of datapoints selected by the active learner.
In Figure \ref{fig:dists}, MC-Dropout, Max-Margin and Max-Entropy strategies are evaluated on CIFAR100 under \textit{canonical label shift}.
By analyzing the uniformity bias and the rate of convergence to the target distribution, we can observe that MALLS exhibits a unique sampling bias which cannot be explained away as simply a class-balancing bias.
This indicates that MALLS may be successful in recovering information from distorted uncertainty estimates.
 \begin{figure}[bthp]
        \centering
         \setlength{\tabcolsep}{-0.2pt}
        \begin{tabular}{ccc}
           \includegraphics[height=4cm]{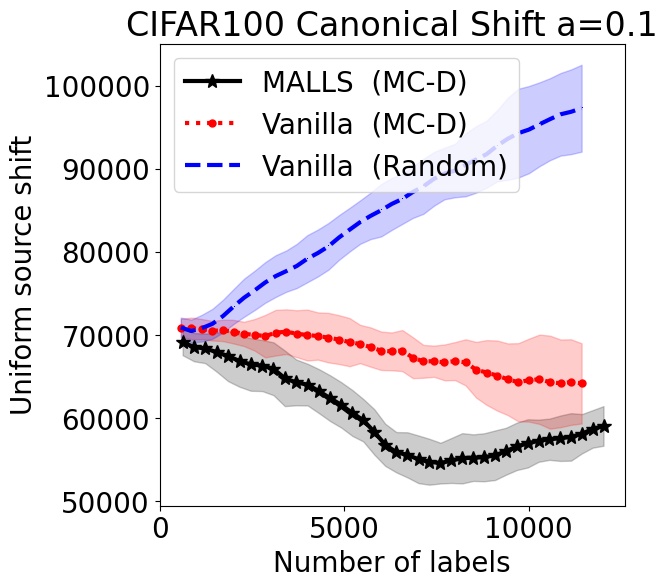}&
           \includegraphics[height=4cm]{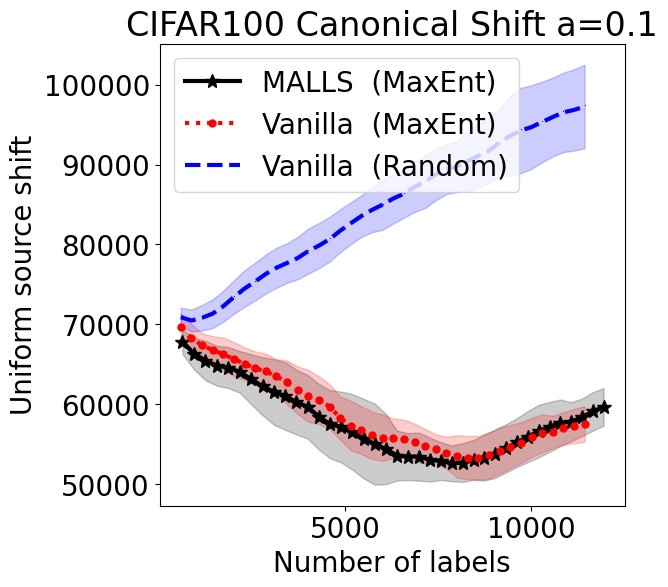}&
           \includegraphics[height=4cm]{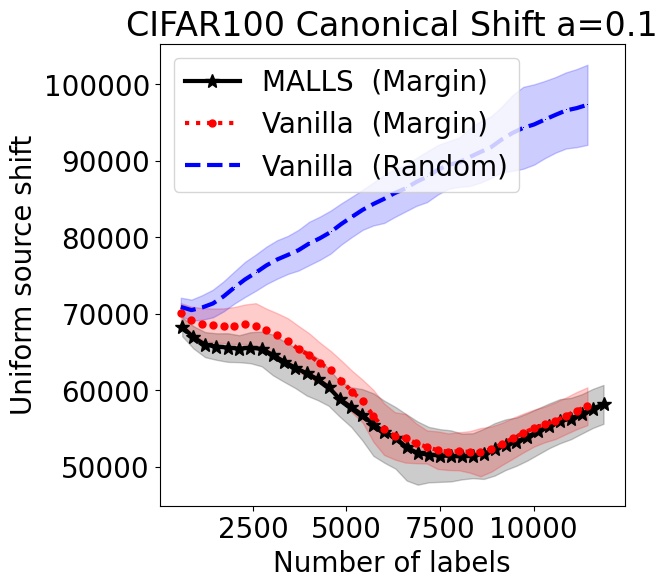} \\
           \includegraphics[height=4cm]{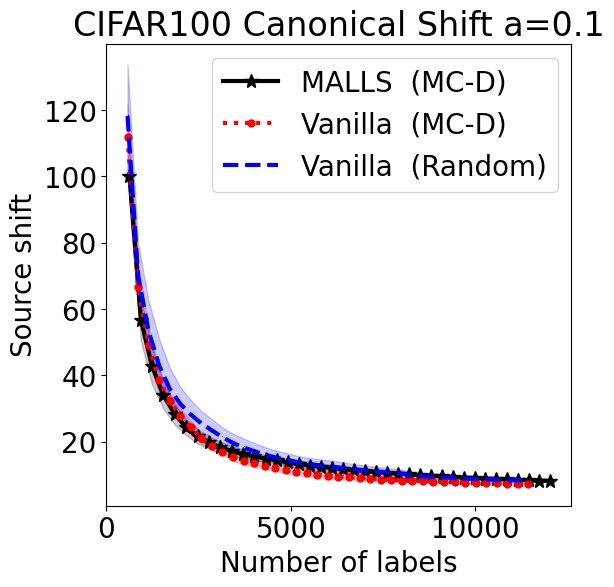}&
           \includegraphics[height=4cm]{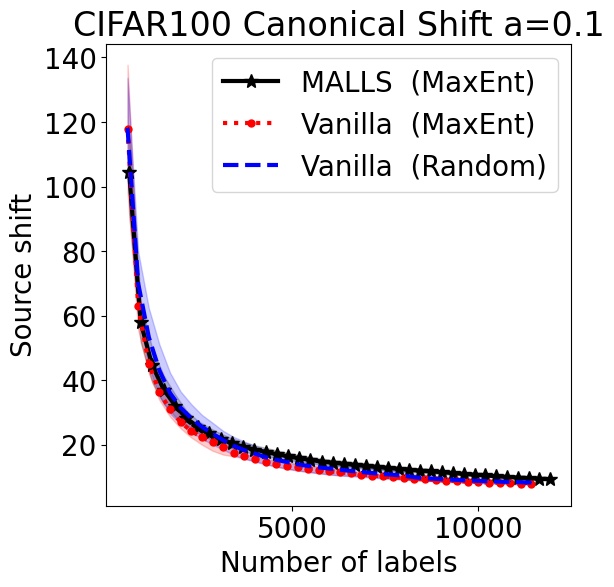}&
           \includegraphics[height=4cm]{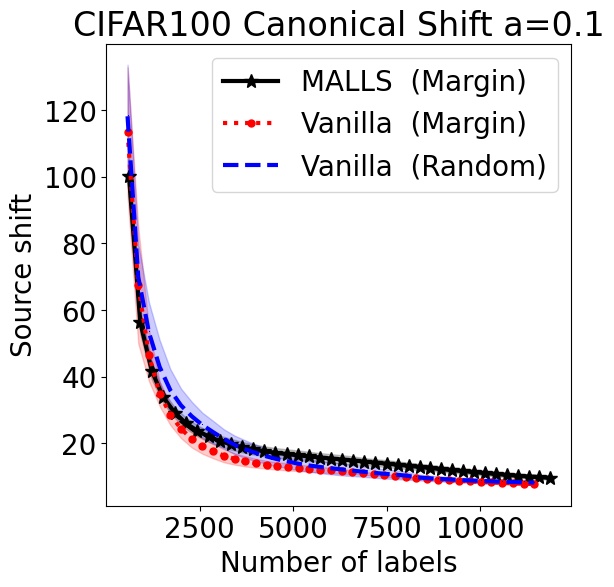}
       \end{tabular}
         \caption{ Average L2 distance between labeled class distribution and uniform/target distribution with 95\% confidence intervals on 10 runs of experiments on CIFAR100 in the \textit{canonical label shift} setting.
         MALLS (denoted by ALLS) converges to the target label distribution slower than vanilla active learning but with a similar uniform sampling bias.
         This suggests MALLS leverages a sampling bias different from that of vanilla active learning or naive class-balanced sampling.
         }
         \label{fig:dists}
    \end{figure}

\section{Experiment Details}
We list our detailed experimental settings and hyperparameters which are necessary for reproducing our results.
Across all experiments, we use a stochastic gradient descent (SGD) optimizer with base learning rate $0.1$, finetune learning rate $0.02$, momentum rate $0.9$ and weight decay $5\mathrm{e}{-4}$.
We also share the same batch size of $128$ and RLLS \cite{azizzadenesheli_regularized_2019} regularization constant of $2\mathrm{e}{-6}$ across all experiments.
As suggested in our analysis, we employ a uniform medial distribution to achieve a balance between distance to the target and distance to the source distributions.
For computational efficiency, all experiments are conducted with minibatch-mode active learning.
In other words, rather than retraining models upon each additional label, multiple labels are queried simultaneously.
Table \ref{tab:stats} lists the specific hyperparameters for each experiment, categorized by dataset.
Table \ref{tab:settings} lists the specific parameters of simulated label shifts (if any) created for individual experiments.
Figure numbers reference figures in the main paper and appendix.
``Dir'' is short for Dirichlet distribution, ``Inh'' is short for inherent distribution, and ``Uni'' is short for uniform distribution.

\begin{table}[ht]
    \centering
    \begin{center}
     \begin{tabular}{||c c c c c c ||} 
     \hline
     Dataset & Model & \# Datapoints & Epochs (init/fine) & \# Batches & \# Classes \\ [0.5ex] 
     \hline\hline
     NABirds1 & Resnet-34 & 30,000 & 60/10 & 20 & 21\\
     \hline
     NABirds2 & Resnet-34 & 30,000 & 60/10 & 20 & 228 \\
     \hline
     CIFAR & Resnet-18 & 40,000 & 80/10 & 40 & 10 \\
     \hline
     CIFAR100 & Resnet-18 & 40,000 & 80/10 & 40 & 100 \\
     \hline
    \end{tabular}
    \end{center}
    \caption{Dataset-wide statistics and parameters}
    \label{tab:stats}
\end{table}

\begin{table}[ht]
    \centering
    \begin{center}
     \begin{tabular}{||c c c c c c c ||} 
     \hline
     Figure & Dataset & Warm Ratio & Source Dist & Target Dist & Canonical? & Dirichlet $\alpha$ \\ [0.5ex] 
     \hline
     \hline
     \ref{fig:alls}(a) & MNIST & 0.1 & Dir & Dir & Yes & 0.1 \\
     \hline
     \ref{fig:alls}(b) & CIFAR & 0.4 & Dir & Dir & Yes & 0.4 \\
     \hline
     \ref{fig:big}(a-b) & CIFAR100 & 0.4 & Dir & Dir & Yes & 0.1 \\
     \hline
     \ref{fig:big}(c-d) & NABirds1 & 1.0 & Inh & Inh & No & N/A \\
     \hline
     \ref{fig:cifar}(a-b) & CIFAR & 0.3 & Dir & Dir & Yes & 0.7 \\
     \hline
     \ref{fig:cifar}(c) & CIFAR & 0.3 & Dir & Dir & Yes & 0.7 \\
     \hline
     \ref{fig:cifar}(d) & CIFAR100 & 0.4 & Dir & Dir & Yes & 0.1 \\
     \hline
     \ref{fig:alphas}(a) & CIFAR100 & 0.4 & Dir & Dir & Yes & 3.0 \\
     \hline
     \ref{fig:alphas}(b) & CIFAR100 & 0.4 & Dir & Dir & Yes & 0.7 \\
     \hline
     \ref{fig:alphas}(c) & CIFAR100 & 0.4 & Dir & Dir & Yes & 0.4 \\
     \hline
     \ref{fig:alphas}(d) & CIFAR100 & 0.4 & Dir & Dir & Yes & 0.1 \\
     \hline
     \ref{fig:ablation}(a) & CIFAR100 & 0.4 & Dir & Uni & No & 1.0 \\
     \hline
     \ref{fig:ablation}(b) & CIFAR100 & 0.3 & Uni & Dir & No & 0.1 \\
     \hline
     \ref{fig:ablation}(c-d) & CIFAR100 & 0.4 & Dir & Dir & Yes & 0.1 \\
     \hline
     T\ref{tab:nabirds}(g-i) & NABirds1 & 1.0 & N/A & Dir & No & 0.1 \\
     \hline
     \ref{fig:dists} & CIFAR100 & 0.4 & Dir & Dir & Yes & 0.1 \\
     \hline
    \end{tabular}
    \end{center}
    \caption{Label Shift Setting Parameters (in order of paper)}
    \label{tab:settings}
\end{table}

\end{document}